\documentclass[letterpaper]{article} % DO NOT CHANGE THIS
\usepackage{aaai25}  % DO NOT CHANGE THIS
\usepackage{times}  % DO NOT CHANGE THIS
\usepackage{helvet}  % DO NOT CHANGE THIS
\usepackage{courier}  % DO NOT CHANGE THIS
\usepackage[hyphens]{url}  % DO NOT CHANGE THIS
\usepackage{graphicx} % DO NOT CHANGE THIS
\usepackage{natbib}  % DO NOT CHANGE THIS AND DO NOT ADD ANY OPTIONS TO IT
\usepackage{caption} % DO NOT CHANGE THIS AND DO NOT ADD ANY OPTIONS TO IT
\usepackage{algorithm}
\usepackage{algorithmic}
\usepackage[symbol]{footmisc}
\usepackage{newfloat}
\usepackage{listings}
\DeclareCaptionStyle{ruled}{labelfont=normalfont,labelsep=colon,strut=off} % DO NOT CHANGE THIS
\floatstyle{ruled}
\newfloat{listing}{tb}{lst}{}
\floatname{listing}{Listing}
\usepackage{amsmath}
\usepackage{amssymb}
\usepackage{amsthm}
\usepackage{stmaryrd}
\usepackage{xcolor}
\usepackage{mathtools}
\usepackage{xfrac}
\usepackage{caption}
\usepackage{subcaption}
\usepackage{xspace}
\usepackage[outline]{contour}
\usepackage{float}
\usepackage{fontawesome}
\usepackage{booktabs}
\usepackage{multirow} 
\usepackage{bm}
\usepackage{dsfont}
\usepackage{tikz,pgf}
\usetikzlibrary{shapes}
\usetikzlibrary{automata,positioning,arrows.meta}
\usepackage{siunitx}
\usepackage{nicefrac}
\usepackage{enumitem}
\usepackage{makecell}
\usepackage{stmaryrd}
\usepackage{pifont}% http://ctan.org/pkg/pifont

\usepackage{xr} % For external references
\usepackage{ifthen}

\newtheorem{theorem}{Theorem}

\newtheorem{lemma}[theorem]{Lemma}
\newtheorem{assumption}{Assumption}

\newtheorem{example}{Example}
\newtheorem*{example*}{Example}

\theoremstyle{definition}
\newtheorem{definition}{Definition}

\mathchardef\mhyphen="2D % Define a "math hyphen"

\newcolumntype{M}[1]{>{\centering\arraybackslash}p{#1}}

\newcommand{\DP}{\mathtt{DP}}
\newcommand{\EO}{\mathtt{EqOpp}}

\newcommand{\FinShield}{\textsf{FinHzn}\xspace}
\newcommand{\StaticDP}{\textsf{Static-Fair}\xspace}
\newcommand{\StaticBAR}{\textsf{Static-BW}\xspace}
\newcommand{\Dyn}{\textsf{Dynamic}\xspace}

\newcommand{\expe}{\mathbb{E}}
\newcommand{\prob}{\mathbb{P}}

\newcommand{\RR}{\mathbb{R}}
\newcommand{\BB}{\mathbb{B}}
\newcommand{\NN}{\mathbb{N}}

\newcommand{\len}[1]{|#1|}

\newcommand{\set}[1]{\{#1\}}

\newcommand{\spec}{\varphi}

\newcommand{\numer}{\mathtt{num}}
\newcommand{\den}{\mathtt{den}}

\newcommand*{\indicator}[1]{\mathbf{1}\left\lbrace#1\right\rbrace}

\newcommand{\cost}{\mathit{cost}}

\newcommand*{\distrset}{\Delta}

\newcommand{\aG}{\mathcal{G}}

\newcommand{\aX}{\mathcal{X}}
\newcommand{\aZ}{\mathcal{Z}}

\newcommand{\aY}{\mathcal{Y}}

\newcommand{\cX}{x}
\newcommand{\cY}{y}

\newcommand{\sgen}{\theta}

\newcommand{\sagent}{\rho}
\newcommand{\sshield}{\pi}
\newcommand{\sShield}{\Pi}
\newcommand{\sShieldFeas}{\Pi_{\mathtt{fair}}}
\newcommand{\sShieldFeasPeriodic}{\Pi_{\mathtt{fair\mhyphen per}}}
\newcommand{\sShieldFeasBounded}{\Pi_{\mathtt{BW}}}
\newcommand{\sShieldFeasDyn}{\Pi_{\mathtt{fair\mhyphen dyn}}}
\newcommand{\costset}{\mathbb{C}}
\newcommand{\Trfeas}{\mathtt{FT}}
\newcommand{\Trbalance}{\mathtt{BT}}
\newcommand{\AccRate}[2]{\mathtt{AR^{#1}}(#2)}
\newcommand{\welfare}[1]{\mathtt{WF}^{#1}}

\usepackage{calc}
\usepackage{graphicx} % Ensure graphicx package is included
\newlength\myheight
\newlength\mydepth
\settototalheight\myheight{Xygp}
\title{Fairness Shields: Safeguarding against Biased Decision Makers}

\author {
    Filip Cano\textsuperscript{\rm 1},
    Thomas A. Henzinger\textsuperscript{\rm 2},
    Bettina K{\"o}nighofer\textsuperscript{\rm 1},
    Konstantin Kueffner\textsuperscript{\rm 2},
    Kaushik Mallik\textsuperscript{\rm 3*}
}
\affiliations {
    \textsuperscript{\rm 1}Graz University of Technology, 
    \textsuperscript{\rm 2}Institute of Science and Technology Austria (ISTA),
    \textsuperscript{3}IMDEA Software Institute\\
    \{filip.cano, bettina.koenighofer\}@iaik.tugraz.at, 
    \{tah, konstantin.kueffner\}@ist.ac.at,
    kaushik.mallik@imdea.org
}

\usepackage{bibentry}
\newboolean{includeappendices}
\setboolean{includeappendices}{true}

\begin{document}

\maketitle

\begin{abstract}
As AI-based decision-makers increasingly influence human lives, it is a growing concern that their decisions are often unfair or biased with respect to people's sensitive attributes, such as gender and race. %crucial to ensure that their decisions are fair and unbiased.
Most existing bias prevention measures provide probabilistic fairness guarantees in the long run, and it is possible that the decisions are biased on specific instances of short decision sequences. 
We introduce \emph{fairness shielding}, where a symbolic decision-maker---the fairness shield---continuously monitors the sequence of decisions of another deployed black-box decision-maker, and makes interventions so that a given fairness criterion is met while the total intervention costs are minimized.
%
% a novel neurosymbolic procedure to guarantee fairness of a deployed black-box decision-maker on every finite run of a given length. The fairness shield monitors and minimally intervenes in the decision-maker’s decisions to ensure that fairness criteria are met either within a bounded horizon or periodically, while also minimizing the costs associated with such interventions as specified by a given cost function. 
We present four different algorithms for computing fairness shields, among which one guarantees fairness over fixed horizons, and three guarantee fairness periodically after fixed intervals.
Given a distribution over future decisions and their intervention costs, our algorithms solve different instances of bounded-horizon optimal control problems with different levels of computational costs and optimality guarantees. %, ensuring fairness with minimal expected costs.
Our empirical evaluation demonstrates the effectiveness of these shields in ensuring fairness while maintaining cost efficiency across various scenarios. 
\end{abstract}

\section{Introduction}

\footnote[1]{\noindent Part of the research was done when the author was at ISTA.}

With the increasing popularity of machine learning (ML) in human-centric decision-making tasks, including banking~\cite{liu2018delayed} and college admissions~\cite{oneto2020fairness}, it is a growing concern that the decision-makers often show biases based on sensitive attributes of individuals, like gender and race~\cite{dressel2018accuracy,obermeyer2019dissecting,scheuerman2019computers}.
Therefore, mitigating biases in ML decision-makers is an important problem and an active area of research in AI.

A majority of existing bias prevention techniques use \emph{design-time} interventions, like pre-processing the training dataset~\cite{kamiran2012data,calders2013unbiased} or tailoring the loss function used for training~\cite{agarwal2018reductions,berk2017convex}.
We propose \emph{fairness shielding}, the first \emph{run-time} intervention procedure for safeguarding fairness of already deployed decision-makers.

Fairness shields consider fairness from the \emph{sequential} decision-making standpoint, where decisions are made on individuals appearing sequentially over a period of time, and each decision may be influenced by those made in the past.
While the classical fairness literature typically evaluated fairness in a single round of decision, the sequential setting has been shown to better captures real-world decision making problems~\cite{zhang2021fairness}. 
Among the works on fairness in the sequential setting, most prior works aimed at achieving fairness in the long run~\cite{hu2022achieving}. 
Recently, the \emph{bounded-horizon} and the \emph{periodic} variants have been proposed, which are often more realistic as regulatory bodies usually assess fairness after bounded time or at regular intervals, 
such as yearly or quarterly~\cite{oneto2020fairness}.

Our fairness shields guarantee bounded-horizon and periodic fairness of deployed unknown decision-makers by monitoring each decision and minimally intervening if necessary. Fairness is guaranteed on \emph{all} runs of the system, whereas existing algorithms guarantee fairness only \emph{on average} over all runs, leaving individual runs prone to exhibit biases~\cite{pmlr-v235-alamdari24a}.

\begin{figure}[t]
    \centering
    \begin{tikzpicture}
        \draw[fill=black!10!white]   (-1,0)   rectangle   (1.2,1)   node[pos=0.5, align=center]    {Sub-symbolic\\classifier};
        \draw[fill=yellow!35!white]   (5,0)   rectangle   (6.8,1) node[pos=0.5,align=center]  {Fairness\\ shield};

        \draw[->]   (1.2,0.3)   --  node[below,align=center]    {intervention cost}  (5,0.3);
        \draw[->]   (1.2,0.6)   --  node[align=center,above]    {recommended decision\\[-0.15cm] {\scriptsize \texttt{accept?}/\texttt{reject?}}}    (5,0.6);

        \node   (a)  at   (2.5,-2)   {\fbox{\Huge\faMale}};
        \node   (b) [left=0.01cm of a]    {\Huge\faFemale};
        \node  (c) [left=0.01cm of b]   {\Huge\faFemale};
        \node[black!75!white] (d) [left=0.01cm of c]    {\Huge\faMale};
        \node[black!50!white] (e) [left=0.01cm of d]    {\Huge\faMale};
        \node[black!25!white] (f) [left=0.01cm of e]    {\Huge\faMale};

        \draw[->]   (a)  -- (2.5,-0.6)   --  (0.1,-0.6) --  (0.1,0) node[align=center]    at    (-0.4,-0.6)  {input\\ features\\[-0.15cm] {\scriptsize \texttt{gender},  \texttt{age}, ...}};
        \draw[->]   (a) --  (2.5,-0.6)  -- node[below,align=center]  {sensitive feature\\[-0.15cm] {\scriptsize \texttt{gender}}} (5.2,-0.6)  --  (5.2,0);

        \node[align=center] (x)   at  (4.2,-2)  {final decision\\[-0.15cm] {\scriptsize \texttt{accept/reject}}};
        \draw[->]   (5.9,0)   --  (5.9,-2)    --  (x);
    \end{tikzpicture}
    \caption{The operational diagram of fairness shields.}
    \label{fig:shield-schematic}
%     \begin{tikzpicture}
%     % Draw the shield outline
%     \draw[thick]
%         % Start from the top-left point
%         (-1.5, 2) -- % top left to top right
%         (1.5, 2) -- % top right to right curve start
%         (1.5, 0) % straight down to bottom right
%         .. controls (1, -2.5) and (0.5, -3.5) .. (0, -4) % right curve to bottom point
%         .. controls (-0.5, -3.5) and (-1, -2.5) .. (-1.5, 0) % left curve to bottom left
%         -- cycle; % close the shape
% \end{tikzpicture}
\end{figure}
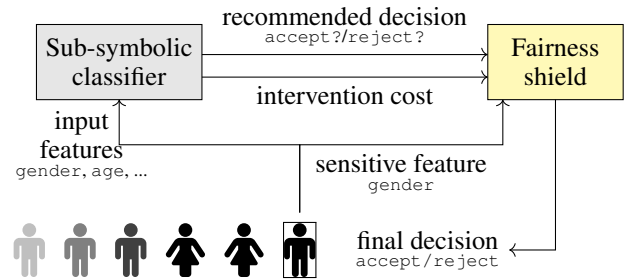

The basic functionality of fairness shields is depicted in Fig.~\ref{fig:shield-schematic}.
We assume a bounded-horizon or periodic fairness property is given, with a known time horizon or period, respectively.
For each individual appearing in sequence, the fairness shield observes the sensitive attribute, the classifier's recommendation, and the cost of changing that recommendation to a different value, where the cost is assumed to be either provided by the decision-maker or pre-specified as constants. 
The shield then makes the final decision, ensuring that the given fairness criteria will be fulfilled while the associated total intervention cost is minimized.

\begin{example}[Running example - Bilingual team]
\label{ex:running example}
Consider the task of building a customer service team in a bilingual country where both languages, $A$ and $B$, hold official status. 
To ensure high-quality service, it is essential to maintain a balanced representation of competent speakers of both languages. 
To achieve this, the company enforces a policy requiring that the difference between the number of employees proficient in each language must not exceed $10\%$ of the total team size. 
The hiring happens sequentially, where candidates apply sequentially and decisions are made on them before seeing the future candidates.
The screening of applicants is undertaken by an automated decision-making system, such as an ML model, which was designed with or without the fairness consideration in mind.
A fairness shield can be deployed, which will monitor and intervene the decisions at runtime to guarantee that the final team is linguistically balanced as required, and yet the deviations from the decision-maker's recommendations are minimal.
\end{example}

% \paragraph{Fairness shielding.} We propose algorithms that \emph{guarantee bounded-horizon or periodic fairness on every run of the system}.
% We propose a neurosymbolic approach that equips a data-driven classifier with a symbolic decision-maker, called a \emph{fairness shield}, as illustrated in Fig.~\ref{fig:shield-schematic}.
% The fairness shield needs to monitor the classifier's decisions and minimally intervene  to ensure that the specified fairness criteria are met. For each individual, the fairness shield observes the sensitive attribute, the classifier's recommendation, and the cost of deviating from that recommendation. The shield then makes the final decision, ensuring that fairness criteria are upheld within a bounded horizon or periodically while minimizing the costs associated with altering the classifier's output.

%With this information, the shield issues a decision on whether to follow the recommendation of the ML model or contravene it. 

%The fairness shield ensures that it is minimally invasive towards the classifier's decisions, while guaranteeing that the final decisions on every run satisfy the given bounded-horizon fairness property or, under mild assumption, the given periodic fairness property.
%\todo{Either state the assumptions or kill the assumption part.}

\paragraph{Computation of fairness shields.} 
Fairness shields are computed by solving bounded-horizon optimal control problems, which incorporate a \emph{hard fairness constraint} and a \emph{soft cost constraint} designed to discourage interventions. 
For the hard fairness constraint, we consider the empirical variants of standard group fairness properties, like demographic parity and equal opportunity. We require that the \emph{empirical bias remains below a given threshold} with the bias being measured either at the end of the horizon or periodically.

For the soft cost constraint, it is assumed that the shield receives a separate cost penalty for each decision modification. The shield is then required to minimize the total expected future cost, either over the entire horizon or within each period.
The definition of cost is subjective and varies by application. 
% Two natural possibilities include:
% In our experiments, we consider two options: 
In our experiments, we consider constant costs that simply discourage high \emph{number} of interventions.
Future works will include more fine-grained cost models, such as costs proportionally varying with the classifier's confidence, thereby discouraging interventions on high-confidence decisions and possibly resulting in higher final utilities.

% (a)~constant cost, which reduces the number of interventions, and (b)~cost proportional to the classifier's confidence, which limits changes to low-confidence decisions, thus minimally affecting the decision-maker's utility.

For shield computation, we assume that the distribution over future decisions (of the classifier) and costs are known, either from the knowledge of the model or \emph{learned} from queries.
Note that even if the distribution is learned and imprecise, the fairness guarantees provided by the shield remain unaffected; only the cost-optimality may be compromised.
Fairness shields are computed through dynamic programming:
while the straightforward approach would require exponential time and memory,
we present an efficient abstraction for the dynamic programming algorithm that reduces the complexity to only \emph{polynomial time}.

\paragraph{Types of fairness shields.}
We propose four types of shields: (i)~\FinShield, (ii)~\StaticDP, (iii)~\StaticBAR, and (iv)~\Dyn shields. 
\FinShield is specific to the bounded-horizon problem, ensuring fairness in every run while being cost-effective. 
The other three are suited for the periodic setting, 
guaranteeing fairness under mild assumptions on how rarely individuals from each group will appear in a period (formalized in Sec.~\ref{sec:unbounded}). 
 \StaticDP and \StaticBAR reuse a statically computed \FinShield shield for each period, while \Dyn shields require online re-computation of shields at the start of each period.

% In Tab.~\ref{tab:extensions-comparison} we summarize the features of these shields, 
% the ratio of their assumption satisfaction, and the success ratio of the shields in producing fair runs. 

\paragraph{Experiments.} 
We empirically demonstrate the effectiveness of fairness shielding on various ML classifiers trained on well-known datasets.
While unshielded classifiers often show biases, their shielded counterparts are fair in \emph{every} run in the bounded-horizon setting and in most runs in the periodic setting 
(fairness may not be guaranteed in runs that don't meet the rareness assumption). 
We demonstrate that the utility loss incurred due to shielding is small in most cases.
It increases as the bias threshold gets smaller and decreases if the classifier was already trained to be fair.

\section{Shielding Fairness}
\label{sec:shielding-fairness}

\noindent\textbf{Data-driven classifier.}
Suppose we are given a population of individuals partitioned into groups $a$ and $b$, where $\aG=\set{a,b}$ are the \emph{sensitive features}, like race, gender, or language, for example.
Consider a data-driven classifier that at each step samples one individual from the population, and outputs a \emph{recommended decision} from the set $\BB=\set{1,0}$ along with an intervention cost from the finite set $\costset\subset \RR_{\geq 0}$. % in case the decision is changed.
As convention, decisions ``1'' and ``0'' will correspond to ``accept'' and ``reject,'' respectively.
We assume that the sampling and classification process gives rise to a given \emph{input distribution} $\sgen\in \distrset(\aX)$, where the set $\aX\coloneqq (\aG\times\BB\times\costset)$ is called the \emph{input space}.
The non-sensitive features of individuals are hidden from the input space because they are irrelevant for shielding.
We will assume that $\sgen$ is given; in App.~\ref{sec:cost_alignment}, we discuss extensions to cases when $\sgen$ is to be estimated from data.

\begin{example}[Continuation of Ex.~\ref{ex:running example}]
\label{ex:running_example2}
    In the bilingual team example, an individual is represented by a tuple $(g, z) \in \aG \times \mathcal{Z}$, 
    where $\aG = \{a, b\}$ denotes the language that the candidate speaks,
    and $\mathcal{Z}$ encompasses all non-sensitive features relevant to evaluating a candidate's suitability for the job,
    such as years of experience, relevant education, and so on.
    For simplicity, we assume that a candidate is proficient in only one of the two languages.
    
    The company uses a classifier 
    $f \colon \aG \times \mathcal{Z} \to \BB \times \costset$, 
    which outputs a preliminary decision for each candidate (accept or reject) along with a cost associated with altering that decision. 
    The cost reflects the classifier’s confidence: candidates who are clearly good or bad incur a high cost for decision changes, 
    while borderline candidates can have their decisions reversed at a lower cost.
\end{example}

\smallskip
\noindent\textbf{Shields.}
A \emph{shield} is a symbolic decision-maker---independent from the classifier---and selects the \emph{final decision} from the \emph{output space} $\aY \coloneqq \BB$ after observing a given input from $\aX$, and possibly accounting for past inputs and outputs.
% The decision set of the shield will be called the \emph{output space}, denoted as $\aY = \BB$.
% The output $1\in\aY$ will represent that the individual is accepted and $0\in\aY$ will represent that the individual is rejected.
Formally, a shield is a function $\sshield\colon \left(\aX\times\aY\right)^*\times \aX \to \aY$, and its bounded-horizon variants are functions of the form $\left(\aX\times\aY\right)^{\leq t}\times \aX \to \aY$, for a given $t$.
We will write $\sShield$ and $\sShield^t$ to respectively denote the set of all shields and the set of bounded-horizon shields with horizon $t$.
The \emph{concatenation} of a sequence of shields $\sshield_1,\sshield_2,\ldots\in \sShield^t$ is a shield $\sshield$, such that for every trace $\tau$, if $\tau$ can be decomposed as $ \tau\tau'$ with $|\tau|=jt$ for some $j$ and $\tau'<t$, then $\sshield(\tau,x) \coloneqq \sshield_{j+1}(\tau',x)$.

% For a given shield $\sshield\in \sShield$, each time $\sshield$'s decision differs from the agent's recommendation, the shield needs to pay the associated cost specified in the agent's output; more on cost will follow shortly.

\smallskip
\noindent\textbf{Shielded sequential decision making.}
We consider the sequential setting where inputs are sampled from $\sgen$ one at a time, and the shield $\sshield$ needs to produce an output without seeing the inputs from the future.
Formally, at every time $i=1,2,\ldots$, let $x_i=(g_i, r_i, c_i)$ be the input appearing with the probability $\sgen(\cX_i)>0$, and let the shield's output be $y_i=\sshield((x_1,y_1),\ldots,(x_{i-1},y_{i-1}),x_i)$.
% denote the observed individual, $r_i$ be  the agent's recommendation, $c_i$ be the cost, and $d_i$ be the  which appear with the probability $\sgen(\cX_i,r_i,c_i)$.
% Next, the shield outputs $d_i = \sshield(h,(x_i,r_i,c_i))$, where $h\in (\aX\times\aY\times\aY)^{i-1}$ is the sequence of past events.
% is denoted as $\cX_i\in\aX$, the agent's output is denoted as $(r_i,c_i)\in\support(\sagent(\cX_i))$, and the shield's decision is denoted as $d_i = $.
The resulting finite sequence $\tau = (x_1,y_1),\ldots,(x_t,y_t)$ is called a \emph{trace} induced by $\sgen$ and $\sshield$, and the integer $t$ is called the \emph{length} of the trace, denoted as $\len{\tau}$; the notation $\Trfeas^t_{\sgen,\sshield}$ will denote the set of every such trace. %, and we will omit $\sgen$ and $\sshield$ when they are unambiguous.
% We write $\Trfeas^t_{\sgen,\sshield}$ to denote the set of every trace of length $t$ induced by $\sgen$ and $\sshield$; when the agent and the shield are clear from the context, we will simply use $\Trfeas^t$ instead.
%Clearly, $\Trfeas^t$ is a strict subset of $(\aX\times\aY)^t$, since there are input-output sequences for which either $\sgen(\cX_i)=0$ or $y_i=\sshield((x_1,y_1),\ldots,(x_{i-1},y_{i-1}),x_i)$ for some $i$.
For every $t$, the probability distribution $\sgen$ and the shield $\sshield$ induce a probability distribution $\prob(\cdot;\sgen,\sshield)$ over the set $(\aX\times\aY)^t$ as follows.
For every trace $\tau\in (\aX\times\aY)^t$, if $\tau\in\Trfeas^t_{\sgen,\sshield}$ then $\prob(\tau;\sgen,\sshield) \coloneqq \prod_{i=1}^t \sgen(\cX_i)$
and otherwise $\prob(\tau;\sgen,\sshield) \coloneqq 0$.
%
%
% The \emph{concatenation} of the traces $\tau$ and $\tau'$ is the trace $\tau\tau'$ whose length is $\len{\tau}+\len{\tau'}$. 
Given a prefix $\tau$, the probability of observing the trace $\tau\tau'$, for some $\tau'\in (\aX\times\aY)^*$, is $\prob(\tau'\mid \tau;\sgen,\sshield) = \prob(\tau\tau';\sgen,\sshield)/\prob(\tau;\sgen,\sshield)$.
%\begin{align*}
%    \prob(\tau'\mid \tau;\sgen,\sshield) = \frac{\prob(\tau\tau';\sgen,\sshield)}{\prob(\tau;\sgen,\sshield)}.
%\end{align*}
(The statistical dependence of $\tau'$ on $\tau$ is due to $\sshield$'s history-dependence.)

\smallskip
\noindent\textbf{Cost.}
Let $\tau = (\cX_1,\cY_1),\ldots,(\cX_t,\cY_t)$ be a trace of length $t$, where $\cX_i = (g_i,r_i,c_i)$.
At time $i$, the shield pays the cost $c_i$ if its output $y_i$ is different from the recommended decision $r_i$.
The \emph{total} (intervention) \emph{cost incurred by the shield} on $\tau$ up to a given time $s\leq t$ is $\cost(\tau;s) \coloneqq \sum_{i=1}^s c_i\cdot\indicator{r_i\neq y_i}$.
%\begin{align*}
%    \cost(\tau;s) \coloneqq \sum_{i=1}^s c_i\cdot\indicator{r_i\neq y_i}.
%\end{align*}
The cost incurred up to time $t$ (the length of $\tau$) is simply written as $\cost(\tau)$, instead of $\cost(\tau;t)$.

For a given time horizon $t$, we define the expected value of cost after time $t$ as $\expe[\cost;\sgen,\sshield,t] \coloneqq \sum_{\tau\in (\aX\times\aY)^t} \cost(\tau)\cdot \prob(\tau;\sgen,\sshield)$, 
%\begin{align*}
%    \expe[\cost;\sgen,\sshield,t] \coloneqq \sum_{\tau\in (\aX\times\aY)^t} \cost(\tau)\cdot \prob(\tau;\sgen,\sshield),
%\end{align*}
and if additionally a prefix $\tau$ is given, the conditional expected cost after time $t$ (from the end of $\tau$) is $ \expe[\cost\mid \tau;\sgen,\sshield,t] \coloneqq \sum_{\tau'\in (\aX\times\aY)^t} \cost(\tau')\cdot \prob(\tau'\mid\tau;\sgen,\sshield)$.
%\begin{align*}
%    \expe[\cost\mid \tau;\sgen,\sshield,t] \coloneqq \sum_{\tau'\in (\aX\times\aY)^t} \cost(\tau')\cdot \prob(\tau'\mid\tau;\sgen,\sshield),
%\end{align*}

% We generalize the definitions of cost and expected cost to time intervals.
% Given a trace $\tau$ of length $t$ and given a time interval $[t_1;t_2]$ with $t_2<t$, we write $\cost(\tau;t_1,t_2)$ to denote the cost incurred only in the interval $[t_1;t_2]$, i.e., $\cost(\tau;t_1,t_2) \coloneqq \cost(\tau;t_2)-\cost(\tau;t_1)$, and moreover, we define $\expe[\cost;\sgen,\sshield,t_1,t_2]$ as follows: 
% First, given a prefix trace $\bar \tau$ of length $t_1$, define 
% $\expe[\cost;\sgen,\sshield,t_1,t_2, \bar \tau]\coloneqq \sum_{\tau\in (\aX\times\aY)^{t_2}} \cost(\bar\tau \tau;t_1,t_2)\cdot \prob(\bar\tau \tau;\sgen,\sshield)$. 
% Since the probability distribution $\sgen$ is fixed, it can be shown that for every pair of prefix traces $\bar{\tau},\bar{\tau}'$ of length $t_1$,  $\expe[\cost;\sgen,\sshield,t_1,t_2, \bar \tau] = \expe[\cost;\sgen,\sshield,t_1,t_2, \bar \tau'] = \expe[\cost;\sgen,\sshield,t_2-t_1]$.  
% Therefore, we write $\expe[\cost;\sgen,\sshield,t_1,t_2]$ to simply denote $\expe[\cost;\sgen,\sshield,t_2-t_1]$.
% \KM{@Filip: Check the new definition.}
% \FC{Actually, I think it does depend on the trace.}
% \todo{Rewrite this in term of conditional probability with $\prob$. To do for Kaushik}

\begin{example}[Continuation of Ex.~\ref{ex:running_example2}]
\label{ex:running_example3}
    The shield $\pi$ is an element external to the classifier. 
    It takes the language group of the candidate and the classifier's recommendation as inputs and has the authority to issue a final accept/reject decision. 
    If the shield's decision differs from the classifier's,
    the incurred cost is as specified by the classifier.
    The shield's inputs are the features of candidates, the classifier's decisions, and the costs, and the input distribution is assumed to be known in advance.
    % Given a distribution of future candidates and the classifier's decisions on them,
    % the shield’s input follows a distribution $\theta \in \distrset(\aX)$.
    
    Note that, from the shield's perspective,
    the distribution of non-sensitive features is unimportant,
    as these features are already processed by the data-driven classifier and summarized into a single cost value. 
    By sampling individuals from the candidate pool and processing them through both $f$ and $\pi$, 
    we obtain a trace $\tau$ that records the individuals and their decisions. 
    This trace encapsulates the results of the hiring process, including the linguistic distribution of hired candidates and the total cost incurred by the shield.
\end{example}

\smallskip
\noindent\textbf{Fairness.}
We model (group) \emph{fairness properties} as functions that map every finite trace to a real-valued \emph{bias} level through intermediate statistics. 
% We consider group fairness properties that can be empirically evaluated using group-specific input-output statistics of a given trace, formalized as follows. 
A \emph{statistic} $\mu$ maps each finite trace $\tau$ to the values of a finite set of counters, represented as a vector in $\NN^p$, where $p$ is the number of counters.
The \emph{welfare} for group $g\in\set{a,b}$ is a function $\welfare{g}\colon \NN^p\to \RR$.
When $\mu$ is irrelevant or clear, we will write $\welfare{g}(\tau)$ instead of $\welfare{g}(\mu(\tau))$.
A fairness property $\varphi$ is an aggregation function mapping $(\welfare{a}(\tau),\welfare{b}(\tau))$ to a real-valued \emph{bias}.
Tab.~\ref{tab:fairness properties} summarize how existing fairness properties, namely demographic parity (DP)~\cite{dwork12}, disparate impact (DI)~\cite{feldman2015certifying}, and equal opportunity (EqOpp)~\cite{hardt2016equality} can be cast into this form.
% Following examples illustrate our formalism on empirical variants of existing fairness properties.

Estimating EqOpp requires the ground truth labels of the individuals be revealed after the shield has made its decisions on them.
To accommodate ground truth, we 
introduce the set $\aZ = \set{0,1}$, such that traces are of the form $\tau = (x_1,y_1,z_1),\ldots,(x_t,y_t,z_t)\in (\aX\times\aY\times\aZ)^*$, where each $z_i$ is the ground truth label of the $i$-th individual.
The shield is adapted to $(\aX\times\aY\times\aZ)^*\times\aX\to\aY$, where the set $\aZ$ is treated as another input space and the probability distribution $P(\aZ=z_i\mid \aX=x_i)$ is assumed to be available.

\begingroup
\renewcommand{\arraystretch}{1.1} % Default value: 1
\begin{table}
    \centering
    \begin{tabular}{M{1cm}  M{2.2cm}  M{1cm}  M{2.7cm}}
    \toprule
        {Name} & {Counters} & $\welfare{g}$ & $\spec$ \\
        \midrule
         DP & $n_a,n_{a1},n_b,n_{b1}$ & $n_{g1}/n_g$ & $|\welfare{a}(\tau)-\welfare{b}(\tau)|$ 
         % if $n_a\neq 0$,  $n_b\neq 0$, otherwise~$0$
         \\
         % \hline
         DI & $n_a,n_{a1},n_b,n_{b1}$ & $n_{g1}/n_g$  & $\left| 
     \welfare{a}(\tau)\div\welfare{b}(\tau)  \right|$ 
     % if $n_a\neq 0$, $n_{b}\neq 0$, $n_{b1}\neq 0$, otherwise~$0$ 
     \\
     % \hline
         EqOpp & $n_a',n_{a1}',n_b',n_{b1}'$ & $n_{g1}'/n_g'$ & $|\welfare{a}(\tau)-\welfare{b}(\tau)|$ 
         % if $n_a'\neq 0$, $n_b'\neq 0$, otherwise~$0$
         \\
         % \hline
         \bottomrule
    \end{tabular}
    \caption{Empirical variants of fairness properties: For $g\in\set{a,b}$, the counters $n_g$ and $n_{g1}$ represent the total numbers of individuals from group $g$ who appeared and were accepted, respectively. Counters $n_g'$ and $n_{g1}'$ denote the total numbers of appeared and accepted individuals whose ground truth labels are ``$1$.'' 
    If a welfare value is undefined due to a null denominator, we set $\spec = 0$.
    }
    \label{tab:fairness properties}
\end{table}

\endgroup

\begin{example}[Continuation of Ex.~\ref{ex:running_example3}]
\label{ex:running_example4}
    In the bilingual team example, the welfare of a linguistic group $g$ is defined as the fraction of the team proficient in language $g$, which is the empirical variant of DP. 
    A more nuanced interpretation considers the welfare of group $g$ as the fraction of accepted candidates among those proficient in language $g$. 
    This measure accounts for the possibility that the linguistic distribution of the population may not be evenly split. 
    If one language is more prevalent in the target population,
    the hired team should proportionally include more members proficient in that language. 
    The fairness property derived from this measure corresponds to an empirical version variant of EqOpp.
\end{example}

\smallskip
\noindent\textbf{Bounded-horizon fairness shields.}
From now on, 
we use the convention that $\sgen$ is the input distribution,
$\varphi$ is the fairness property, 
and $\kappa$ is the \emph{bias threshold}.
Let $T$ be a given time horizon.
The set $\sShieldFeas^{\sgen,T}$  of \emph{fairness shields over time $T$} is the set of every shield that fulfills $\spec(\cdot)\leq \kappa$ after time $T$, i.e., 
    % \begin{multline*}
    %     \sShieldFeas^t \coloneqq \set{ \sshield\in\sShield \mid \forall \tau\in (\aX\times\aY)^t\;.\; \\
    %     \prob(\tau;\sgen,\sshield) > 0 \implies \DP(\tau) \leq \kappa}.
    % \end{multline*}
%\begin{equation}\label{eq:fair-shields-set}
$\sShieldFeas^{\sgen,T} \coloneqq \set{ \sshield\in\sShield^T \mid \forall \tau\in \Trfeas^{T}_{\sgen,\sshield}\;.\; \spec(\tau) \leq \kappa}$.
%\end{equation}
%
%
%
% For simpler notation, we will usually drop ``$\sgen$'' when it is clear from the context.
% \KM{I suggest not to do this.}
%
We now define optimal bounded-horizon fairness shields as below.

% ----------- Old problem statement -----------
% \begin{problem}[Bounded-horizon fairness-shield synthesis]\label{prob:bounded}
%     Let $\sgen\in \distrset(\aX)$ be a given joint input distribution, $\kappa>0$ be a given DP threshold, let $T>0$ be a time horizon.
%     Compute a fairness-shield $\sshield^*$ that minimizes the expected cost accumulated over time $[1;T]$, i.e., 
%     \begin{align*}
%         \sshield^* \coloneqq \arg\min_{\sshield\in\sShieldFeas^T} \expe[\cost;\sgen,\sshield,T].
%     \end{align*}
% \end{problem}
% -------------------------------------------

\begin{definition}[\FinShield shields]
\label{def:finshield}
    % Let $\sgen$ be the input distribution, $\kappa>0$ be the bias threshold, and 
    Let $T>0$ be the time horizon.
    A \FinShield shield is the one that solves:
    \begin{align}\label{eq:finite horizon shield}
        \sshield^* \coloneqq \arg\min_{\sshield\in\sShieldFeas^{\sgen,T}} \expe[\cost;\sgen,\sshield,T].
    \end{align}
\end{definition}

% \FC{Suggestion: maybe we could define here BAR shields. They are a natural alternative definition, we can say in the recursion how to adapt the recursion to be BAR, and lay the work for the periodic section.}

\smallskip
\noindent\textbf{Periodic fairness shields.}
\FinShield shields stipulate that fairness be satisfied at the end of the given horizon.
However, in many situations, it may be desirable to ensure fairness not only at the end of the horizon but also at intermediate points occurring at regular intervals.
For instance, a human resources department required to maintain a fair distribution of employees over the course of a quarter
might also need to ensure a similar property for every yearly revision.
This type of fairness is referred to as \emph{periodic fairness} in the literature~\cite{pmlr-v235-alamdari24a}.
% We adapt the definition of fairness shields to be able to cope with periodic fairness specifications.
For this class of fairness properties, we define the set of $T$-periodic fairness shields as $\sShieldFeasPeriodic\coloneqq \set{ \sshield\in\sShield \mid \forall m\in \NN \;.\; \forall \tau\in \Trfeas^{mT}_{\sgen,\sshield}\;.\; \spec(\tau) \leq \kappa}$.

% Up until now, we have defined our shields with the objective of satisfying a certain DP constraint after a fixed time horizon $T$. 
% With our proposed method, shield synthesis has a cost that grows in $T$, so from certain time horizon on, 
% the method described becomes computationally unfeasible.

% One natural question is what can be done for traces longer than $T$, assuming we can only afford computing shields with time horizon $T$.
% Ideally, we would like to build shields that are periodically fair.
\begin{definition}[Optimal $T$-periodic fairness shield]
    \label{def:periodically-fair}
    % Let $\sgen$ be the input distribution, $\kappa>0$ be the bias threshold, and 
    Let $T>0$ be the time period.
    An \emph{optimal $T$-periodic fairness shield} is given by:
    % is a shield $\sshield$ such that for every positive integer $m$, every trace $\tau\in \Trfeas_{\theta, \sshield}^{mT}$ satisfies $\spec(\tau) \leq \kappa$, and moreover, for every $\tau'\in (\aX\times\aY)^{<T}$ and $\cX\in\aX$, we have $\sshield(\tau\tau',\cX)=\sshield^*(\tau',\cX)$ where (local cost-optimality):
    % \begin{align*}
    %      \sshield^* \coloneqq \arg\min_{\sshield\in\sShieldFeas^{mT}} \expe[\cost\mid \tau;\sgen,\sshield,T].
    % \end{align*}
    \begin{align}\label{eq:periodic def.}
        \sshield^* \coloneqq \arg\min_{\sshield\in\sShieldFeasPeriodic}\sup_{\substack{m\in\NN\\ \tau\in \Trfeas_{\sgen,\sshield}^{mT}}} \expe[\cost\mid \tau;\sgen,\sshield,T].
    \end{align}
\end{definition}
% \KM{We do not mention cost optimality anymore which looks fishy.}
Eq.~\eqref{eq:periodic def.} requires fairness at each $mT$-th time (measured from the beginning), and minimizes the maximum expected cost over each period. 
The existence of this minimum remains an open question. 
In Sec.~\ref{sec:unbounded}, we propose three "best-effort" approaches to compute periodically fair shields (under mild assumptions) that are as cost-optimal as possible.

% -------- For future reference -----------
% \begin{problem}[Unbounded-horizon fairness-shield synthesis]\label{prob:unbounded}
%    Let $\sgen\in \distrset(\aX)$ be a given joint distribution of sampling individuals and the output of the agent $\sagent$, and let $\kappa>0$ be a given DP threshold.
%    Compute a shield $\sshield^*$ such that for every large enough $T$ (there exists a $t>0$ such that for every $T\geq t$), $\sshield^*$ is a fairness-shield over time $T$ (i.e., $\sshield^*\in\sShieldFeas^T$) and moreover $\sshield^*$ minimizes the expected cost after time $T$, i.e., 
%    \begin{align*}
%        \expe[\cost;\sgen,\sshield^*,T] = \min_{\sshield\in\sShieldFeas^T} \expe[\cost;\sgen,\sshield,T].
%    \end{align*}
% \end{problem}

%%%%%%
%%%%%% Commenting out the periodic fairness shield

% \begin{problem}[Periodic fairness-shield synthesis]\label{prob:unbounded}
%     Let $\sgen\in \distrset(\aX)$ be a given joint distribution of sampling individuals and the output of the agent $\sagent$, and let $\kappa>0$ be a given DP threshold.
%     Suppose $T$ is a given time period. 
%     Compute a shield $\sshield^*$ such that for every $k\in \NN$, $\sshield^*$ is a fairness-shield over time $kT$ (starting from the beginning of time), and moreover, for every prefix $\tau$ of length $(k-1)T$, it minimizes the expected cost in the following time period, i.e., 
%     \begin{align*}
%         \sshield^* \coloneqq \arg\min_{\sshield\in\sShieldFeas^{kT}} \expe[\cost\mid \tau;\sgen,\sshield,T].
%     \end{align*}
% \end{problem}

\section{Algorithm for \FinShield Shield Synthesis}
\label{sec:synthesis}
We present our algorithm for synthesizing \FinShield shields as defined in Def.~\ref{def:finshield}.
A \FinShield shield $\sshield^*$ computes an output $\cY=\sshield^*(\tau,x)$ for every trace $\tau$ and every input $x$. 
Our synthesis algorithm builds $\sshield^*$ recursively for traces of increasing length, 
using an auxiliary \emph{value function} $v(\tau)$ that represents the minimal expected cost conditioned on traces with prefix $\tau$.
To define $v(\tau)$, we generalize fairness shields with the condition that a certain trace has already occurred. 
Given a time horizon $t$ and a trace $\tau$ (length can differ from $t$), the set of \emph{fairness shields over time $t$ after $\tau$} is defined as:
%\begin{multline*}
$\sShieldFeas^{\sgen,t \mid 
\tau} \coloneqq \set{ \sshield\in\sShield^t \mid \forall \tau'\in (\aX\times\aY)^{t}\;.\; \tau\tau'\in \Trfeas^{\len{\tau}+t}_{\sgen,\sshield}
\implies \spec(\tau\tau') \leq \kappa}$. %;
%\end{multline*}
% we omit $\sgen$ whenever clear from the context.
Then $v(\tau)$ is given by:
\begin{align*}
    v(\tau) \coloneqq \min\limits_{\pi\in \sShieldFeas^{\sgen,(T-\len{\tau}) \mid\tau}} \expe[\cost\mid \tau;\sgen,\sshield,T-|\tau|].
\end{align*}
For every trace $\tau$ and every input $x\in \aX$, 
the optimal value of the shield is 
$\pi^*(\tau,x) = \arg\min_{\cY\in\aY} v(\tau, (\cX,\cY))$.

In Sec. 3.1, we present a recursive dynamic programming for computing $v(\tau)$, whose complexity grows exponentially with the length of $\tau$.
In Sec. 3.2, we present an efficient solution using only the $p$ counters defining the fairness property, thus solving the synthesis problem in $\mathcal{O}(T^p\cdot |\aX|)$-time.
From now on, we present the main ideas in the text, and refer the reader to App.~\ref{sec:detailed-proofs} for detailed proofs of all results.

% our method has two steps
% \begin{enumerate}
%     \item Compute and store the minimal expected costs $v(\tau, (x,0))$ and $v(\tau, (x,1))$.
%     \item Assign $\pi^*(\tau,x) = \arg\min_{\cY\in\aY} v(\tau,(\cX,\cY))$.
% \end{enumerate}
% In Sec. we detail how the computations in step 1 can be done recursively. In Sec. we explain how this method can be abstracted to compute both $v$ and $\pi^*$ on relevant counters instead of full traces, making the procedure much more efficient.

\subsection{Recursive Computation of $v(\tau)$}
\label{sec:naive recursive solution}
\subsubsection{Base case.} Let $T$ be the time horizon and $\tau$ be a trace of length $T$.
Since the horizon is reached if $\spec(\tau)\leq \kappa$ then the expected cost is zero because fairness is already satisfied and no more cost needs to be incurred, whereas if $\spec(\tau)>\kappa$, the expected cost is infinite,
because no matter what cost is paid fairness can no longer be achieved. Formally,
\begin{equation}\label{eq:v-basecase}
 v(\tau) = 
        \begin{cases}
            0   & \spec(\tau)\leq\kappa,\\
            \infty  &   \text{otherwise}.
        \end{cases}
\end{equation}

\subsubsection{Recursive case.}
Let $\tau$ be a trace of length smaller than $T$.
The probability of the next input being $x=(g,r,c)$ is $\sgen(x)$, and the shield decides to output $\cY$ that either agrees with the recommendation $r$ (the case $y=r$) or differs from it (the case $y\neq r$)---whichever minimizes the expected cost. 
When $y=r$, then the trace becomes $(\tau, (x,y=r))$. Therefore, no cost is incurred and the total cost remains the same as $v(\tau,(x,y=r))$. 
When $y\neq r$, the trace becomes $(\tau, (x,y\neq r))$. Thus, the incurred cost is $c$ and the new total cost becomes $c+v(\tau,(x,y=r))$. Therefore
\begin{equation}\label{eq:v-recursion}
    \! v(\tau) = \sum_{\cX = (g,r, c)\in\aX} \sgen(\cX)\cdot \min\left\lbrace\begin{matrix}
        v(\tau,(x,y=r)),\\
        v(\tau,(x,y\neq r)) + c
    \end{matrix}\right\rbrace.
\end{equation}
Eqs.~\eqref{eq:v-basecase} and \eqref{eq:v-recursion} can be used to recursively compute $v(\tau)$ for every $\tau$ of length up to $T$, and the time and space complexity of this procedure is $\mathcal{O}(|\aX\times\aY|^T)$.
% With Eqs.~\ref{eq:v-basecase} and~\ref{eq:v-recursion}, 
% we can compute $v(\tau)$ for every $\tau$ using a standard dynamic programming setup, where values of $v(\tau)$ are computed with either Eq.~\ref{eq:v-basecase} or Eq.~\ref{eq:v-recursion} and stored to be used in computations of $v(\tau')$ for traces $\tau'$ of shorter length.
The correctness of Eq.~\eqref{eq:v-recursion} is formally proven in App.~\ref{sec:detailed-proofs}, Lem.~\ref{lem:v-recursion}.

% \subsection{Synthesis using Dynamic Programming}
\subsection{Efficient Recursive Computation of $v(\tau)$}
\label{sec:synthesis:efficient}
 % This would be very costly, as the number of traces increases exponentially with $T$. We propose to do it more efficiently.
We now present an efficient recursive procedure for computing \FinShield shields that runs in polynomial time and  space. %, as opposed to the exponential blow-up of the algorithm in Sec.~\ref{sec:naive recursive solution}. 
The key observation is that $\spec$ is a fairness property that depends on $\tau$ through a statistic that uses $p$ counters.
Consequently, $v(\tau)$ in Eq.~\eqref{eq:v-basecase} and Eq.~\eqref{eq:v-recursion} depend only on counter values, not on exact traces. 
This allows us to define our dynamic programming algorithm over the set of counter values taken by the statistic $\mu$.
Let $R_{\mu,T} \subseteq \mathbb{N}^p$ be the set of values the statistic $\mu$ can take from traces of length at most $T$.
We have the following complexity result.
%Therefore, we obtain the following result.% (proof in App.~\ref{sec:detailed-proofs}).
\begin{theorem}\label{thm:bounded-horizon shield synthesis-bis}
    The bounded-horizon shield-synthesis problem 
 can be solved in $\mathcal{O}(|R_{\mu, T}|\cdot |\aX|)$-time and $\mathcal{O}(|R_{\mu,T}|\cdot |\aX|)$-space.
\end{theorem}
% \KM{Refer to appendix for the proof? App.\ref{sec:detailed-proofs}}

In most fairness properties like DP and EqOpp, 
the range of values they can take is $R_{\mu, T} = [0, T]^p$, 
where $p$ is the number of counters ($p=4$ for DP, and $p=5$ for EqOpp), 
making the complexity polynomial in the length of the time horizon.

\section{Algorithms for Periodic Shield Synthesis}
\label{sec:unbounded}

We present algorithms for computing periodic fairness shields for a broad subclass of group fairness properties, termed \emph{difference of ratios} (DoR) properties.
A statistic $\mu$ is \emph{single-counter} if it maps every trace $\tau$ to a single counter value, i.e., $\mu(\tau) \in \NN$, and \emph{additive} if $\mu(\tau\tau') = \mu(\tau) + \mu(\tau')$ for any traces $\tau$ and $\tau'$. A group fairness property $\spec$ is DoR if (a) for each group $g$, $\welfare{g}(\tau) = \numer^g(\tau)/\den^g(\tau)$, where $\numer^g(\tau)$ and $\den^g(\tau)$ are additive single-counter statistics, and (b) $\spec(\tau) = |\welfare{a}(\tau) - \welfare{b}(\tau)|$.
Many fairness properties, including DP and EqOpp, are DoR, though DI is not because it violates the condition (b). For DoR fairness properties, we propose two approaches for constructing periodic fairness shields: \emph{static} and \emph{dynamic}, and we explore their respective strengths and weaknesses.

\subsection{Periodic Shielding: The Static Approach}
In the static approach, 
a periodic shield is obtained 
by \emph{concatenating infinitely many identical copies of a statically computed bounded-horizon shield} $\sshield$, synthesized with the time period $T$ as the horizon.
We present two ways of computing $\sshield$ so that its infinite concatenation is $T$-periodic fair.
% We present two possibilities.
% We formalize this below.
%
% \begin{definition}[Static generalization of a shield]
%     Let $\sshield\in \sShield^T$ be a shield computed for a given time horizon $T$.
%     The \emph{static generalization of $\sshield$} is the shield $\sshield^s\colon (\aX\times\aY)^*\times\aX\to\aY$ defined as follows. 
%     For a trace $\tau$ and an input $x\in \aX$, let $\tau_1$ and $\tau_2$ be the two (unique) traces such that $|\tau_1| = jT$, for some integer $j$, $|\tau_2|=t<T$ and $\tau=\tau_1\tau_2$.
%     Then 
%     \begin{equation}
%         \sshield^s(\tau, x) \coloneqq \sshield(\tau_2, x).
%     \end{equation}    
% \end{definition}
% \KM{We need a better name for this.}
%
%

\subsubsection{Approach I: \StaticDP shields.}

% We will say that an approach for periodic shielding is \emph{static} when a single shield $\pi$ computed for time horizon $T$ can be reused for traces of arbitrary length. 

% In this approach, we simply  use the static generalization $\sshield^s$ of the optimal fairness shield $\sshield$ 
% as defined in Def.~\ref{def:finshield}, i.e., for $\sshield\in \sShieldFeas^T$.
% We define \StaticDP shields as follows.

\begin{definition}[\StaticDP shields]
    A shield is called \StaticDP if it is the concatenation of infinite copies of a \FinShield shield (from Def.~\ref{def:finshield}).
\end{definition}
%
% This approach corresponds to using the shield until time $T$, 
% resetting the counters, using it anew until time $2T$, and so on. 
Unfortunately, \StaticDP shields do not always satisfy periodic fairness.
Consider a trace $\tau=\tau_1\ldots\tau_m$ for an arbitrary $m>0$, generated by a \StaticDP shield, such that each segment $\tau_i$ is of length $T$.
It follows from the property of \FinShield shields that $\spec(\tau_i) \leq \kappa$ for each individual $i$.
However, $T$-periodic fairness may be violated because $\spec(\tau)$ need not be bounded by $\kappa$.
A concrete counter-example for DP is shown below; more examples are in App.~\ref{sec:counterexamples-static-fair}.
%in Tab.~\ref{ex:counter-example-periodic-naive-bis}, and a more general family of counter-examples is presented in App.~\ref{sec:counterexamples-static-fair}.

\begin{example}
\label{ex:counter-example-periodic-naive-bis}
    Consider DP with $0<\kappa<1-2/T$.
    Suppose $\tau_1$ and $\tau_2$ are traces of length $T$ such that for $\tau_1$, $n_a=1$,$n_b=T-1$, and $n_{a1}=n_{b1}=0$, and for $\tau_2$, $n_a=n_b=T$, $n_{a1}=T$, and $n_{b1}=1$. Then $\spec(\tau_1)=\spec(\tau_2)=0$ (fair), but $\spec(\tau_1\tau_2) = |(T-1)/T-1/T |=1-2/T > \kappa$ (biased).
\end{example}

An important feature of these counter-examples is the excessive skewness of appearance rates across the two groups.
We show that \StaticDP shields are $T$-periodic fair if the appearance rates of the two groups are equal at every period.
% We formalize this below.

% \begin{table}
%     \centering
%     \begin{tabular}{M{0.5cm} | M{0.8cm} M{0.8cm} M{0.8cm} M{0.5cm} | M{1.2cm} | M{1cm}}
%     \toprule
%           & $n_a$ & $n_{a1}$ & $n_b$ & $n_{b1}$ & DP ($\spec$) &  {\scriptsize $\spec \leq \kappa$?} \\
%           \midrule
%          $\tau_1$ & $1$ & $0$ & $T-1$ & $0$ & $0$ & \cmark\\
%          $\tau_2$ & $T-1$ & $T-1$ & $1$ & $1$ & $0$ & \cmark \\
%          % \midrule
%          $\tau_1\tau_2$ & $T$ & $T-1$ & $T$ & $1$  & $1-2/T$ & \xmark\\
%          \bottomrule
%     \end{tabular}
%     \caption{Counter-example  showing that \StaticDP shields may not be periodically fair for DP.
%     Suppose the bias threshold is $0 < \kappa < 1-2/T$.
%     The traces $\tau_1,\tau_2$ fulfill DP but their concatenation does not.
%     % \FC{We can eliminate the last column if short on space}
% %     are of length $T$ such that $\spec(\tau_1) = \spec(\tau_2) = 0$, 
% % but $\spec(\tau_1\tau_2) = 1-2/T$.
% %     For $T$ large enough, $\spec(\tau_1\tau_2)$ can get arbitrarily close to 1.
% }
%     \label{fig:counter-example-periodic-naive-bis}
% \end{table}

\begin{definition}[Balanced traces]
    Let $\mu^a,\mu^b\colon (\aX\times\aY)^*\to \NN$ be a pair of group-dependent (single-counter) statistics, $T>0$ be a given time horizon, and
    $N\leq T/2$ be a given integer.
    A trace $\tau$ of length $T$ is \emph{$N$-balanced with respect to} $\mu^a$ and $\mu^b$ if both $\mu^a(\tau) \geq N$ and $\mu^b(\tau)\geq N$; the set of all such traces is written $\Trbalance^T(\mu^a,\mu^b,N)$.
    % We denote the set of all $N$-balanced traces of length $t$ as $\Trbalance_N^t$.
\end{definition}

\begin{theorem}[Conditional correctness of \StaticDP shields]
\label{thm:static shield with bounded DP}
    Let $\spec$ be a DoR fairness property. Consider a \StaticDP shield $\sshield$, and let $\tau = \tau_1 \ldots \tau_m \in \Trfeas^{mT}_{\sgen, \sshield}$ be a trace such that $\len{\tau_i} = T$ for all $i \leq m$. If $\den^a(\tau_i) = \den^b(\tau_i)$ for every $i \leq m$, then the fairness property $\spec(\tau) \leq \kappa$ is guaranteed.
\end{theorem}

While the condition in Thm.~\ref{thm:static shield with bounded DP} appears conservative, 
we show in App.~\ref{sec:counterexamples-static-fair} that it is in fact tight for the worst-case satisfaction of DP, 
in the sense that for every $\kappa$, there exist $m$ and $\lfloor(T-1)/2\rfloor$-balanced traces $\tau_1,\dots,\tau_m$ such that $\spec_\DP(\tau_i)\leq \kappa$ for each $i$, but $\spec_\DP(\tau_1\dots\tau_m) > \kappa$.
% there are examples 
% for any magnitude of $\kappa$ of $\lfloor(T-1)/2\rfloor$-balanced traces $\tau_1,\dots,\tau_m$ 
% that satisfy $\spec_\DP(\tau_i)\leq \kappa$, but $\spec_\DP(\tau_1\dots\tau_m) > \kappa$.
%
However, these are worst-case scenarios and are ``uninteresting.'' 
In our experiments, \StaticDP shields
fulfill periodic fairness in a majority of cases even if the condition of Thm.~\ref{thm:static shield with bounded DP} is violated.

% Looking at the example in Table~\ref{fig:counter-example-periodic-naive-bis}, 
% one may wonder how common are these cases. 
% In Theorem~\ref{thm:K(N-K)} (in the appendix)
% we describe a more general family of counter-examples with similar properties.
% In our experiments we also found some counterexamples, 
% although in general we found the na\"ive approach to produce traces
% within the desired values of DP most of the time.

\subsubsection{Approach II: \StaticBAR shields.}
When the condition of Thm.~\ref{thm:static shield with bounded DP} is violated, \StaticDP shields cannot guarantee fairness as the bound on the bias is not closed under concatenation of traces (see Ex.~\ref{ex:counter-example-periodic-naive-bis}).
A stronger property that is closed under concatenation is when a bound is imposed on each group's welfare. 
Let $l,u$ be constants with $0\leq l<u\leq 1$.
A trace $\tau$ has \emph{bounded welfare} (BW) if for each group $g\in\aG$,
$\welfare{g}(\tau)=\numer^g(\tau)/\den^g(\tau)$ belongs to $[l,u]$.
The pair $(l,u)$ will be called \emph{welfare bounds}.
We show that BW is closed under trace concatenations, which depends on the additive property of $\numer^g$ and $\den^g$.
% If a trace $\tau$ has BW for parameters $(l,u)$, 
% then $\spec(\tau) \leq u-l$.
% A bound on the welfare is closed under concatenations.

% Static shielding with bounded DP fails to guarantee DP as DP is not closed under concatenations, as shown already.
% We introduce a stricter property that is closed by concatenation and implies DP. 
% We call this property \emph{bounded acceptance rates}, or \BAR in short, which requires that the acceptance rates of both groups be bounded within a given range $[l,u]$ at the end of the horizon; when $u-l\leq \kappa$, \BAR implies DP (but not vice versa). 
% The pair $(l,u)$ will be called \emph{BAR parameters}.
% \FC{Rewrite, mention $0 \leq l < u \leq 1$}
% A static shield fulfilling \BAR with respect to $[l,u]$ can be easily obtained by adapting Eq.~\eqref{eq:v-basecase} with $v(\tau)=0$ if $\AccRate{g}{\tau}\in [l,u]$ by both $g\in \set{a,b}$, and $v(\tau)=\infty$ otherwise (as usual).

\begin{lemma}
    \label{prop:acc-rates}
    Let $(l,u)$ be given welfare bounds, and $\welfare{g}(\cdot)\equiv \numer^g(\cdot)/\den^g(\cdot)$ for additive $\numer^g,\den^g$.
    For a trace $\tau = \tau_1\dots \tau_m$,
    if for each $i$, %\leq m$, 
    $\welfare{g}(\tau_i) \in [l,u]$, 
    then 
    $\welfare{g}(\tau) \in [l,u]$.
    % \begin{multline*}
    %     \forall i\in [1;n]\;.\; \AccRate{g}{\tau_i} \in [l;u]\\ \implies \AccRate{g}{\tau} \in [l;u].
    % \end{multline*}
    % if for every $i\in\{1\dots n\}$,
    % $\AccRate{g}{\tau_i} \in [l;u]$, then
    % $\AccRate{g}{\tau_1\dots \tau_n} \in [l;u]$. 
\end{lemma}
% \begin{proof}[Proof (sketch)]
%     This result follows from a known fundamental inequality. 
%     Given positive numbers
%     $w, x, y, z$, if $w/x \leq y/z$, then
%     $\frac{w}{x} \leq \frac{w+y}{x + z} \leq \frac{y}{z}$.
%     Details are in App.~\ref{sec:detailed-proofs}.
% \end{proof}

For DoR properties, BW implies fairness when $u-l\leq \kappa$.
Combining this with Lem.~\ref{prop:acc-rates}, we infer that if $\sshield$ is a bounded-horizon shield that fulfills BW on every trace $\tau$ of length $T$ for welfare bounds $(l,u)$ with $u-l\leq \kappa$, then the concatenation of infinite copies of $\sshield$ would be a $T$-periodic fairness shield.
The natural course of action for computing shields 
that fulfill BW is to 
mimic Def.~\ref{def:finshield}, 
replacing the condition on $\spec$ with a condition on welfare. 
However, if we define the set of BW-fulfilling shields as
$
     \sShieldFeasBounded^{\sgen,T}\coloneqq  \set{\sshield\in\sShield \mid \forall \tau \in \Trfeas^T_{\sgen, \sshield} \;.\;
    \forall g\in \set{a,b}\;.\;
    l\leq \welfare{g}(\tau) \leq u }
$,
the set $\sShieldFeasBounded^{\sgen,T}$ can be empty for some $T,l,u$.
Following is an example.
% This is because some traces are incompatible with some bounds on welfare.
\begin{example}\label{ex:bounded-acc-rates}
    Suppose $\welfare{g}(\tau) = n_{g1}/n_g$, where $n_{g1}$ and $n_g$ are the total numbers of accepted and appeared individuals from group $g$ (as in DP).
    Suppose $T=2, l=0.2,u=0.4$.
    It is easy to see that no matter what the shield does, for every $\tau$ of length $2$, $\welfare{g}(\tau)\in \set{0,0.5,1}$.
    Therefore, $\Pi^{2}_{[0.2,0.4]}=\emptyset$.
    % Consider a trace $\tau$ with
    %  $n_{a}(\tau) = 2$. 
    %  The only values that $\AccRate{a}{\tau}$ can take are $0$, $1/2$, and $1$, so a pair of bounds $l=0.2$, $u=0.4$ is impossible to enforce.
\end{example}
The emptiness of $\sShieldFeasBounded^{\sgen,T}$ is due to a large disparity between the appearance rates of individuals from the two groups, which occurs for shorter time horizons and for datasets where one group has significantly lesser representation than the other group.
% like this are practically uninteresting ``edge cases,''
% where members of one group appear too few times.
To circumvent this technical inconvenience, 
we make the following assumption on observed traces.
% \FC{I would not state this as "uninteresting edge cases". It's the same phrasing we use for static-fair shields, and these are much more common than those.}

% We now introduce our second category of static shields, called the \StaticBAR.
% The na\"ive idea would be to follow the same approach as in \StaticDP, but while computing the \FinShield shield, replace $\sShieldFeas^T$ in Eq.~\ref{eq:finite horizon shield} with the following set of shields that guarantee the BW condition (instead of DP):
% \begin{multline*}
%      \sShieldFeasBounded^{T}\coloneqq  \set{\sshield\in\sShield \mid \forall \tau \in \Trfeas^T_{\sgen, \sshield} \;.\; \\ 
%     \forall g\in \set{a,b}\;.\;
%     l\leq \AccRate{g}{\tau} \leq u }.
% \end{multline*}
% Unfortunately, $\sShieldFeasBounded^{T}$ will be empty in general (see Ex.~\ref{?} in App.~\ref{?}), and therefore the \FinShield shield would not exist.
% Luckily, the emptiness of $\sShieldFeasBounded^{T}$ results from practically uninteresting ``edge cases,'' where members of one group need to appear at a low rate and strategically near the end of the horizon to throw the difference in acceptance rates outside the bounds $[l,u]$.
% To circumvent this technical inconvenience, we make the following mild assumption on the traces that we will observe at runtime.

% \begin{assumption}\label{assump:static-BAR}
%     A given trace $\tau\in \Trfeas^T$ is $N$-balanced with $N\geq \left\lceil \frac{1}{u-l}\right\rceil$, for given BW parameters $(l,u)$.
% \end{assumption}

\begin{assumption}\label{assump:static-BAR}
    Let $l,u$ be welfare bounds, and $\tau = \tau_1\ldots\tau_m\in \Trfeas^{mT}_{\sgen,\sshield}$ be a trace with $|\tau_i|=T$ for each $i$.
    Every $\tau_i$ is $N$-balanced w.r.t.\ $\den^a$ and $\den^b$ for $N= \left\lceil 1/(u-l)\right\rceil$.
\end{assumption}

Assump.~\ref{assump:static-BAR} may be reasonable depending on $l$, $u$, $T$, and the input distribution $\theta$. 
Intuitively, for a larger $T$ and a smaller skew of  appearance probabilities for individuals between the two groups, the probability of fulfilling Assump.~\ref{assump:static-BAR} is larger (for a given finite $m$).
In App.~\ref{sec:condition-on-BAR-shields-existing}, 
we quantify this probability as the probability of a sample from a binomial distribution lying between $N$ and $T-N$.
%
% We can finally define shields with a BW guarantee as follows.

% We introduce \StaticBAR shields as follows.

\begin{definition}[\StaticBAR shields]
    Let $l,u$ be given welfare bounds, and $T$ be a given time period.
    A \StaticBAR shield is the concatenation of infinite copies of the shield $\sshield^*$ solving
    \begin{align}\label{eq:static-BAR-shield optimality}
        \sshield^* = \arg\min_{\sshield\in\sShieldFeasBounded^{\sgen,T, N}} \expe[\cost;\sgen,\sshield,T], 
    \end{align}
    where $N= \left\lceil 1/(u-l)\right\rceil$, and
    \begin{multline*}
     \sShieldFeasBounded^{\sgen,T,N}\coloneqq  \set{\sshield\in\sShield \mid \forall \tau \in \Trfeas^T_{\sgen, \sshield} \cap \Trbalance_N^T 
     \\ 
    \forall g\in \set{a,b}\;.\;
    l\leq \welfare{g}(\tau) \leq u }.
\end{multline*}
\end{definition}

In App.~\ref{sec:detailed-proofs} (Lem.~\ref{thm:acc-rates-existence}), we provide a constructive proof showing that $\sShieldFeasBounded^{\sgen,T,N}$ is indeed non-empty when Assump.~\ref{assump:static-BAR} is fulfilled.
% by constructing the shield that keeps $\welfare{g}(\tau)$ just above $l$
% and showing that it also guarantees $\welfare{g}(\tau)\leq u$ when the trace is sufficiently balanced.
This result guarantees that the optimization problem in \eqref{eq:static-BAR-shield optimality} is feasible, 
% $\sshield^*$ always exists, 
and thus \StaticBAR shields are well-defined.
Intuitively, we obtain a ``best-effort'' solution for $\sshield^*$: 
when a trace satisfies Assump.~\ref{assump:static-BAR}, 
$\sshield^*$ guarantees that $\tau$ satisfies BW with minimum expected cost. 
Otherwise, $\sshield^*$ has no BW requirement, 
and thus for traces that violate Assump.~\ref{assump:static-BAR}, the shield will incur zero cost by never intervening.

\paragraph{Synthesis of \StaticBAR shields}
follows the same approach as in Sec.~\ref{sec:synthesis} with Eq.~\eqref{eq:v-basecase} replaced by: 
% we can use the same method as presented in Sec.~\ref{sec:synthesis}, changing the condition in Eq.~\ref{eq:v-basecase} to
 \begin{align*}
 v(\tau) = 
        \begin{cases}
            0   & \tau\notin \Trbalance_N^T\lor \bigwedge_{g\in\set{a,b}} \welfare{a}(\tau)\in[l,u],\\
            \infty  &   \text{otherwise}.
        \end{cases}
\end{align*}

We summarize the fairness guarantee below.

\begin{theorem}[Conditional correctness of \StaticBAR shields]
\label{thm:correctness-staticBW}
    Let $\spec$ be a DoR fairness property.
    Let $l,u$ be welfare bounds such that $u-l\leq \kappa$.
    For a given \StaticBAR shield $\sshield$, 
    let $\tau = \tau_1\ldots\tau_m\in \Trfeas^{mT}_{\sgen,\sshield}$ be a trace with $\len{\tau_i}=T$ for each $i\leq m$.
    If Assump.~\ref{assump:static-BAR} holds, then the fairness property $\spec(\tau)\leq \kappa$ is guaranteed.
\end{theorem}

\subsection{Periodic Shielding: The Dynamic Approach}

While the static approaches repeatedly use one statically computed bounded-horizon shield, the dynamic approach recomputes a new bounded-horizon shield at the beginning of each period, and thereby adjusts its future decisions based on the past biases.
% In the dynamic setting, after every time period the shield is recomputed by accounting for the past inputs and outputs observed on the current trace.
We formalize this below.

\begin{definition}[\Dyn shields]
\label{def:dynamic-shield}
    Suppose we are given a parameterized set of \emph{available} shields $\sShield'(\tau)\subseteq \sShield$ where the parameter $\tau$ ranges over all finite traces.
    A \Dyn shield $\sshield$ is the concatenation of a sequence of shields $\sshield_1,\sshield_2,\ldots$ such that for every trace $\tau\in \Trfeas_{\sgen,\sshield}^{mT}$ with $m\geq 0$, for every $\tau'\in (\aX\times\aY)^{<T}$, and for every input $x\in\aX$, $\sshield(\tau\tau',x) = \sshield_{m+1}(\tau',x)$ where
    \begin{align}\label{eq:dyn-shield optimality}
        \sshield_{m+1} = \arg\min_{\sshield'\in\sShield'(\tau)} \expe[\cost\mid \tau;\sgen,\sshield',T].
    \end{align}
\end{definition}

The set $\sShield'(\tau)$ restricts the available set of shields that can be used for the next period for the given history $\tau$.
A na\"ive attempt for $\sShield'(\tau)$ would be to choose $\sShield'(\tau)=\sShieldFeas^{\sgen,T\mid\tau}$ for every $\tau$, so that fairness is guaranteed at the end of the current period. 
However, there exist histories for which $\sShieldFeas^{\sgen,T\mid\tau}$ would be empty (see Ex.~\ref{ex:buffered} in App.~\ref{sec:reflections}), implying that Eq.~\eqref{eq:dyn-shield optimality} would not have a feasible solution for some $\tau$, and the \Dyn shield would exhibit undefined behaviors.
To circumvent this technical inconvenience, we make the following mild assumption on the set of allowed histories, requiring $\sShield'(\tau)$ to fulfill fairness only if $\tau$ fulfills this assumption.

\begin{assumption}\label{assump:dynamic}
    For a given trace $\tau\in \Trfeas_{\sgen,\sshield}^{jT}$ with $j>0$, every valid suffix $\tau' $ of length $t$ (i.e., $\tau'\in \set{\tau''\in (\aX\times\aY)^T\mid \tau\tau''\in\Trfeas_{\sgen,\sshield}^{(j+1)T}}$) fulfills:
    \begin{align*}
        \frac{1}{\den^{a}(\tau\tau')} + \frac{1}{\den^b(\tau\tau')} \leq \kappa + \spec(\tau).
    \end{align*}
\end{assumption}

The set of shields $\sShield'(\cdot)$ available to the \Dyn shield in Def.~\ref{def:dynamic-shield} is then defined as:
\begin{align*}
    \sShield'(\tau) = \sShieldFeasDyn^{\sgen,T}(\tau) \coloneqq  
    \begin{cases}
        \sShieldFeas^{\sgen,T\mid\tau}    &   \tau \text{ fulfills Assump.~\ref{assump:dynamic}},\\
        \sShield                &   \text{otherwise}.
    \end{cases}
\end{align*}
We prove that $\sShieldFeas^{\sgen,T\mid\tau}$ is non-empty whenever $\tau$ fulfills Assump.~\ref{assump:dynamic}  (see App.~\ref{sec:detailed-proofs}, Lem.~\ref{lem:buff-shield-exists}), implying that $\sShieldFeasDyn^{\sgen,T}(\tau)$ is non-empty for every $\tau$.
Technically, this guarantees that the optimization problem in \eqref{eq:dyn-shield optimality} is feasible and $\sshield_{m+1}$ always exists and \Dyn shields are well-defined.
Intuitively, we obtain a ``best-effort'' solution: If Assump.~\ref{assump:dynamic} is fulfilled then $\sshield_{m+1}$ is in $ \sShieldFeas^{\sgen,T\mid\tau}$ and achieves fairness for the minimum expected cost.
Otherwise, $\sshield_{m+1}$ can be any shield in $\sShield$ that only optimizes for the expected cost; in particular, $\sshield_{m+1}$ will be the trivial shield that never intervenes (has zero cost).

\paragraph{Synthesis of \Dyn shields} involves computing the sequence of shields $\sshield_1,\sshield_2,\ldots$, which are to be concatenated.
We outline the algorithm below.
\begin{enumerate}
    \item Generate a \FinShield shield (Def.~\ref{def:finshield}) $\sshield$ for the property $\spec$ and the horizon $T$. Set $\sshield_1\coloneqq \sshield$.
    \item For $i\geq 1$, let $\sshield$ be the concatenation of the shields $\sshield_1,\ldots,\sshield_i$, and let $\tau\in \Trfeas_{\sgen,\sshield}^{iT}$ be the generated trace. Compute $\sshield_{i+1}$ that uses the same approach as in Sec.~\ref{sec:synthesis} with Eq.~\eqref{eq:v-basecase} being replaced by:
        \begin{align*}
            v(\tau') = 
                \begin{cases}
                    0   &   \spec(\tau\tau')\leq \kappa\\
                    \infty & \text{otherwise.}
                \end{cases}
        \end{align*}
\end{enumerate}
We summarize the fairness guarantee below.
\begin{theorem}[Conditional correctness of \Dyn shields]
\label{thm:correctness-dyn-shields}
    Let $\spec$ be a DoR fairness property.
    Let $\sshield$ be a  \Dyn shield that uses $\sShieldFeasDyn^{\sgen,T}(\cdot)$ as the set of available shields.
    Let $\tau =  \tau_1\ldots\tau_m\in \Trfeas^{mT}_{\sgen,\sshield}$ be a trace with $\len{\tau_i}=T$ for each $i\leq m$.
    Suppose for every $i\leq m$, $\tau_1\ldots\tau_i$ fulfills Assump.~\ref{assump:dynamic}.
    Then the fairness property $\spec(\tau)\leq \kappa$ is guaranteed.
\end{theorem}

\section{Experiments}
\label{sec:experiments}

% We will use three datasets: Adult~\cite{misc_adult_2}, COMPAS~\cite{compas}, and Bank Marketing~\cite{misc_bank_marketing_222}.

% Use benchmark from~\cite{han2024ffb}.

% 

% We demonstrate the effectiveness of fairness shields by testing them with ML classifiers which were trained on different standard data sets using different existing ML algorithms---most of which already safeguard against classification biases.
% Each experiment uses a combination of the following items; details are in App.~\ref{sec:experiments_appendix}.
% \begin{description}
%     \item[Data sets and their sensitive features.] Adult~\cite{misc_adult_2} - race and gender, COMPAS~\cite{compas} - race and gender, German Credit~\cite{dua2017uci} - age and gender, Bank Marketing~\cite{misc_bank_marketing_222} - race.
%     \item[Algorithms for training the classifier.]  DiffDP~\cite{fairmixup}, 
% HSIC~\cite{hsic}, 
% LAFTR~\cite{madras2018learning}, PRemover~\cite{kamishima2012fairness}, and a greedy classifier only minimizing empirical risk (ERM) while disregarding fairness.
%     \item[Fairness properties.] DM and EqOpp.
%     \item[Bias thresholds.] $\kappa\in \set{0.05, 0.1,0.15,0.2}$. 
%     \item[Intervention cost.] Constant. 
% \end{description}

\paragraph{Experimental setup.}
%We demonstrate the effectiveness of fairness shields by testing them with ML classifiers which have the same architecture, are trained on different standard datasets in the fairness literature, and use different existing ML algorithms. 
We performed our experiments on the datasets 
Adult~\cite{misc_adult_2}, COMPAS~\cite{compas}, 
German Credit~\cite{dua2017uci}, 
and Bank Marketing~\cite{misc_bank_marketing_222}.
The protected attributes include race, gender, and age.
% We used neural networks with a fixed architecture, and trained them using existing approaches from the literature (mentioned later).
%The learning algorithms are summarized later in their appropriate places.
We synthesized shields to ensure DP and EqOpp with thresholds $\kappa\in \set{0.05, 0.1,0.15,0.2}$.
%The high variability in our results indicates the contrasting nature of the datasets and the ML models.
We give further details on our experimental setup %, including the characteristics of the ML models
in App.~\ref{sec:experimental_setup_appendix}.

\subsection{\FinShield Shields}

The ML models were trained with DiffDP~\cite{fairmixup}, 
HSIC~\cite{hsic}, 
LAFTR~\cite{madras2018learning},
and PRemover~\cite{kamishima2012fairness}.
As a baseline, we also trained a classifier using empirical risk minimization (ERM). For all models and datasets, \FinShield shields were synthesized with $T=100$ for DP and $T=75$ for EqOpp. Shield synthesis took about 1 second and 30 MB for DP, and 1.5 seconds and 1.3 GB for EqOpp. We present the detailed resource usage in App.~\ref{sec:shield-computation-times}. We compared model performances---with and without shielding---across $30$ simulated runs. The analysis follows.

%BEFORE SHORTENING
%The \FinShield shield was used on ML models which were trained using 
%DiffDP~\cite{fairmixup}, 
%HSIC~\cite{hsic}, 
%LAFTR~\cite{madras2018learning},
%and PRemover~\cite{kamishima2012fairness}.
%We also trained a fifth classifier simply minimizing empirical risk (ERM),
%as a baseline.
%For all ML models and datasets, we synthesized \FinShield shields with $T=100$ for DP and $T=75$ for EqOpp.
%Shield synthesis for DP takes about 1 second and 30 MB, 
%while for EqOpp takes about 1.5 seconds and 1.3 GB.
%We present the detailed resource usage in App.~\ref{sec:shield-computation-times}.
%We compare the performances of the models---with and without shielding---on $30$ simulated runs in each case.
%Following is the analysis.
% \todo{Write down the resource usage here in one sentence, to give a range of computation times and memory usages, and then refer to the appendix if more data is available.}

% In this group of experiments, we investigate the performance of \FinShield shields on a single period. 
% We use a time horizon of $T=100$ for DP and $T=75$ for EqOpp.
% For each setting we synthesized a \FinShield shield and simulated $30$ runs. 

\paragraph{Fairness.}
Unshielded ML models violated bounded-horizon fairness in $44\%$ of the cases for DP and in $65\%$ for EqOpp, while shielded models were always fair at the horizons. Detailed results are in Fig.~\ref{fig:fair-distr} and Tab.~\ref{tab:fair_distr} in App.~\ref{sec:experiments_appendix}.
% We observe that across all settings a sizable portion of unshielded runs violate the fairness constraint at the end of the time horizon:  $43.61\%$ for DP and $65.12\%$ for EqOpp. 
This empirically validates the effectiveness of \FinShield shields. %, which guarantee that every run is fair. 
% FC: I would not include this paragraph in teh main text, I commented it out
% In fact, $50\%$ of the runs finish with a fairness value below $42.19\%$ for DP and $21.23\%$ for EqOpp of the allowed threshold, by contrast the values without shield are $83.19\%$ $176.47$ respectively (see Fig.~\ref{fig:fair_distr} and Tab.~\ref{tab:fair_distr}). 
%
% This, however, comes at a cost. 

%BEFORE SHORTENING
\paragraph{Utility loss.}
Classification utility is measured using classification accuracy. Note that interventions by the fairness shield can reduce this utility. 
We define \emph{utility loss} as the difference in utility between unshielded and shielded runs. Tab.~\ref{fig:utility-comp} shows the average utility loss across all simulations for a threshold of $0.1$. 
We can observe that the median utility loss is smaller when the classifier is trained to be fair, as fewer interventions are needed. In general, utility loss increases as the bias threshold $\kappa$ decreases, with more pronounced differences between classifiers for smaller $\kappa$.

\begin{table}[t]
    \centering
    \small
    \begin{tabular}{l l c c c c}
    \toprule
          & & Recomp. & Assump. & Fairness \\
          & &         & satisfied & satisfied\\
          \midrule
          % \FinShield   & bound.   & no  & - & $100\%$ \\
          \multirow[c]{3}{*}{DP} & \StaticDP  & no & $0.0\%$ & $95.71\%$   \\
          & \StaticBAR  & no & $43.8\%$ & $83.1\%$  \\
          & \Dyn & yes &  $100\%$ & $100\%$ \\
          \midrule
          \multirow[c]{3}{*}{EqOpp} & \StaticDP  & no & $0.0\%$ & $100\%$   \\
          & \StaticBAR & no & $4.1\%$ & $56.4\%$  \\
          & \Dyn & yes &  $49.8\%$ & $100 \%$ \\
          \bottomrule
          %Existence & always & always & conditional & conditional  
    \end{tabular}
    \caption{Comparison of different types of fairness shields.
    % \FC{I propose to have this table in the intro. It's short, it conveys how different shields guarantee and execute fairness, and it is understandable. I think putting the theoretical condition is really difficult to understand at this point.  
    % It would also be difficult to put the theoretical condition with some numbers. Not impossible, but I think this is more informative and takes less space.
    % The numbers right now are garbage, but we can get the real ones.}
    }
    \label{tab:extensions-comparison-bis}
\end{table}

\begin{table*}[t]
   \centering
   \includegraphics[width=\linewidth]{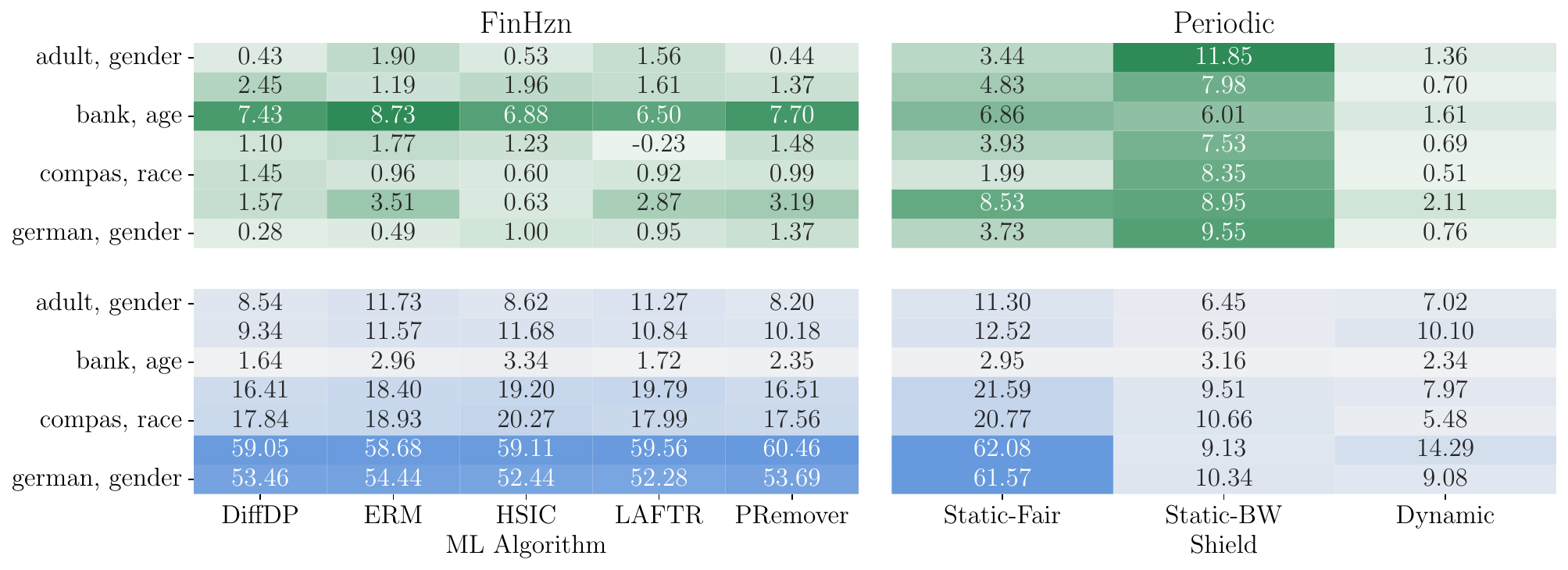}  %
    \caption{Utility loss (in \%) incurred by \FinShield shields for different ML models (left) and by periodic shields on the ERM model (right) for the fairness properties DP (top, green) and EqOpp (bottom, blue). Lighter colors indicate smaller utility loss.
    % \KM{Are the following possible? (a)~Remove the ``-'' signs because in the text we call them utility loss (loss = ``-''). (b)~Instead of ``Finite,'' use the title ``Bounded-horizon (\FinShield shield)'' for the left tables. (c)~Reorder the columns in the table on the right side according to how they appear in the text: \StaticDP, \StaticBAR, and \Dyn. } 
    }
    \label{fig:utility-comp}
\end{table*}

\begin{figure*}[t]
   \centering
   \includegraphics[width=\linewidth]{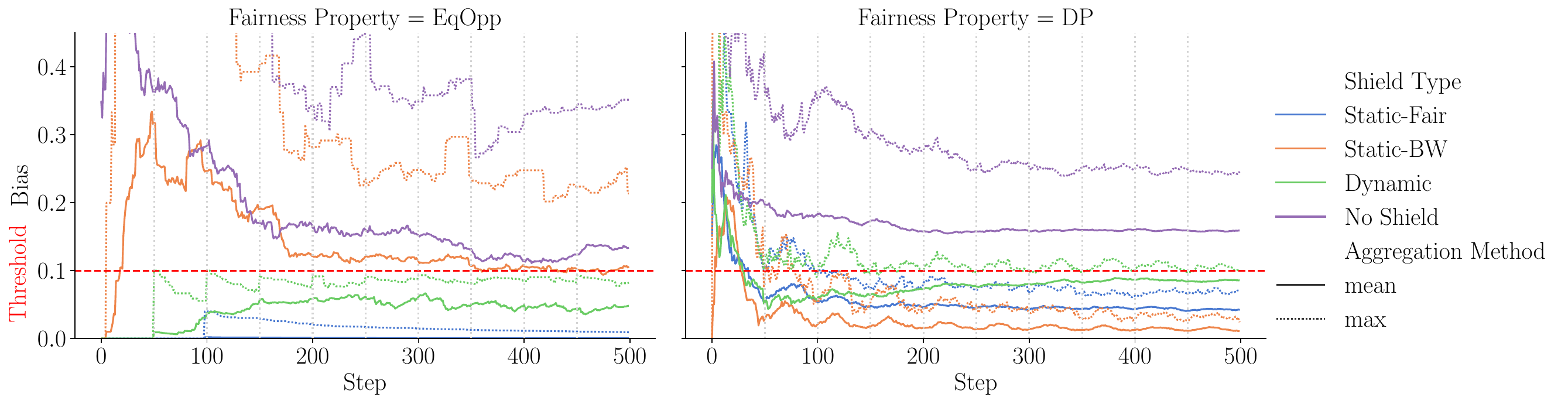}  %
    \caption{
    Variations of bias over time for the ERM classifier on the Adult dataset with and without periodic shielding.
    % The aggregated simulations for the Adult dataset with sensitive attribute gender, the ML algorithm ERM, a DP threshold of 0.1. The right plot excludes period one to highlight the periodic behavior of the shields.
    }
    \label{fig:single-run}
\end{figure*}

% \begin{figure*}[t]
%    \centering
%    \includegraphics[width=\linewidth]{images/heat_util_thr-0.1.pdf}  %
%     \caption{Decrease in utility for DP (left) and EqOpp (right)}
%     \label{fig:heat_single_util_comp}
% \end{figure*}

% \newcommand{\StaticDP}{\textsf{Static-Fair}\xspace}
% \newcommand{\StaticBAR}{\textsf{Static-BW}\xspace}
% \newcommand{\Dyn}{\textsf{Dynamic}\xspace}
% \newcommand{\BAR}{\text{BAR}\xspace}

\subsection{Periodic Shielding}
ML models were trained using the ERM algorithm across all datasets. We synthesized \StaticDP, \StaticBAR, and \Dyn shields with $T=50$ for DP and EqOpp, and simulated them for $10$ periods. Shield synthesis took about 1 second and 30 MB for DP, and 1.5 seconds and 1.3 GB for EqOpp. Detailed resource usage is in App.~\ref{sec:shield-computation-times}. We compared the models' performances---with and without shielding---across $20$ simulated runs. The analysis follows.
% \todo{upd. perf.}

%BEFORE SHORTENING
%\subsection{Periodic Shielding}
%ML models were trained using the ERM algorithm on all datasets.
%In each case, we synthesized \StaticDP, \StaticBAR and \Dyn shields with $T=50$ for DP and EqOpp for $10$ periods. 
%Shield synthesis for DP takes about 1 second and 30 MB, 
%while EqOpp takes about 1.5 seconds and 1.3 GB.
%We present the detailed resource usage in App.~\ref{sec:shield-computation-times}.
%We compare the performances of the models---with and without shielding---on $20$ simulated runs in each case.
%Following is the analysis. \todo{upd. perf.}
% \todo{Write down the resource usage here in one sentence, to give a range of computation times and memory usages, and then refer to the appendix if more data is available.}

\paragraph{Fairness.}
In Tab.~\ref{tab:extensions-comparison-bis}, we present the rates of assumption and fairness satisfaction across all datasets and runs. 
The assumption for \StaticDP (see Thm.\ref{thm:static shield with bounded DP}) never met, 
and the assumption for \StaticBAR (see Thm.~\ref{thm:correctness-staticBW}) is also often violated. 
Nevertheless, both \StaticDP and \StaticBAR still perform well as heuristics, with many runs satisfying the fairness constraint. 
% \StaticBAR shields remain idle upon violation leading to lower satisfaction rates. 
\Dyn shields outperform both.

%BEFORE SHORTENING
%\paragraph{Fairness.}
%In Tab.~\ref{tab:extensions-comparison-bis} we report the rates of assumption satisfaction and fairness satisfaction averaged over all datasets and runs. 
%Since the assumption for \StaticDP is very strict (see Thm.~\ref{thm:static shield with bounded DP}), 
%it is almost never satisfied.
%The assumption for \StaticBAR is also violated in a big proportion of our experiments.
%However, both \StaticDP and \StaticBAR remain suitable heuristics, in the sense that a large portion of the runs still satisfy the fairness constraint.
%The more expensive \Dyn shields prove to be the best performers.

% In Tab.~\ref{tab:extensions-comparison-bis} we can observe that the assumptions for \StaticBAR and \StaticDP are violated frequently decreasing with higher thresholds. Once violated the fairness guarantee is lost. However, both approaches remain suitable heuristics, i.e., a large portion of the runs still satisfy the fairness constraint. As a heuristic $\StaticDP$ outperforms $\StaticBAR$, which can be attributed to its restrictive assumption and the chosen fall-back strategy. 
% Comparing the runs satisfying the assumption we can observe that $50\%$ of the runs finish with a fairness value below $XX \%$ ($XX \%$) for $\StaticDP$, $XX \%$  ($XX \%$) for  $\StaticBAR$ and $XX \%$ ($XX \%$) for $\Dyn$ below the threshold for DP (EqOpp). 
% We can observe that $\StaticBAR$ is by far the most restrictive. This is reflected in the cost.
% \KK{Turn into table.}

\paragraph{Utility loss.}
In Tab.~\ref{fig:utility-comp}, we report the average utility loss across all simulations for each shield for the bias threshold of $0.1$. % and the ML algorithm ERM. %ask filip 
% In general, if the assumptions are satisfied, $\Dyn$ shields incur the least loss and $\StaticBAR$ shields incur the most, which is due to their stricter BW objectives. However, the frequent assumption violations force $\StaticBAR$ shields to often go inactive (and incur no utility loss). Hence, the increased assumption violation rate explains the difference in utility between DP and EqOpp.  and the decreased fairness satisfaction rate.
In general, if the assumptions are satisfied, $\Dyn$ shields incur the least loss and $\StaticBAR$ shields incur the most, which is due to their stricter BW objectives. However, an assumption violation forces both $\Dyn$ and $\StaticBAR$ shields to go inactive incurring no additional utility loss. Therefore, the low utility loss of $\StaticBAR$ shields in EqOpp can be explained by the frequent assumption violations.

\section{Related Work}
\label{sec:related_work}
% \noindent\textbf{Design-time approaches towards fair decision-making.}

Existing work on fairness in AI focuses on how to \emph{specify}, \emph{design}, and \emph{verify} fair decision-making systems. 
Specification involves quantifying fairness across groups~\cite{feldman2015certifying,hardt2016equality} and individuals~\cite{dwork12}. 
Design approaches ensure that decision-makers meet fairness objectives \cite{hardt2016equality,gordaliza2019obtaining,zafar2019fairness,agarwal2018reductions,wen2021algorithms}.
Verification includes static~\cite{albarghouthi2017fairsquare,bastani2019probabilistic,sun2021probabilistic,ghosh2020justicia,meyer2021certifying,li2023certifying} and runtime~\cite{albarghouthi2019fairness,henzinger2023monitoring,henzinger2023dynamic} methods to assess fairness. 
% Our approach combines design and verification, creating fairness shields that are \emph{verified} by \emph{design}.
Our shields are \emph{verified} by \emph{design}, and act as trusted third-party intervention mechanisms, ensuring fairness without requiring knowledge of the AI-based classifier.

Traditionally, fairness is defined using the classifier's output distribution.
However, this can lead to biases over short horizons~\cite{pmlr-v235-alamdari24a}.
To address this, we adopt the recently proposed bounded-horizon fairness properties~\cite{pmlr-v235-alamdari24a}, ensuring decisions remain empirically fair over a bounded horizon.
To the best of our knowledge, our work is the first to provide algorithmic support for guaranteeing bounded-horizon fairness properties.

Sequential decision-making problems have been extensively studied under the umbrella of optimal stopping problems~\cite{shiryaev2007optimal,bandini2018backward,ankirchner2019verification,bayraktar2020optimal,palmer2017optimal,kallblad2022dynamic}.
These works focus on designing policies that approximate those that have perfect foresight about the future. However, statistical properties like fairness are not addressed by existing algorithms in this literature.

\cite{fscd24} proposed sequential decision making algorithms for the general class of finite-horizon statistical properties.
% We built upon their idea in the recursive computation of \FinShield shields, where traces with same counter values (statistically indistinguishable) are combined together to reduce computation costs without changing the output.
They demonstrated that combining statistically indistinguishable traces in dynamic programming reduces computational costs without altering the output. 
 We apply this idea in our \FinShield shields, where traces with identical counter values remain indistinguishable.

\section{Discussion and Future Work}
\label{sec:discussion}
\paragraph{Static vs. dynamic shielding in the periodic setting.}
% The strengths and weaknesses of the three types of periodic shields are as follows.
% \begin{description}
    % \item[Static versus \Dyn.] 
   Static shields are computationally cheaper than \Dyn shields and have no runtime overhead, making them ideal for fast decision-making applications like online ad-delivery~\cite{ali2019discrimination}.
   However, they can't adjust decisions based on the actual history, leading to overly restrictive and frequent interventions---particularly in the long run.
   In contrast, \Dyn shields adapt to historical data, resulting in fewer interventions over time, making them suitable for applications like banking where decision-making can afford longer computation times~\cite{liu2018delayed}.

\paragraph{On the feedback effect in sequential decision-making.}
Decisions that seem fair individually can introduce biases over time as the input distribution $\sgen$ changes based on past actions~\cite{d2020fairness,sun2023algorithmic}. 
Although we assumed a constant $\sgen$ in this paper, our recursive synthesis algorithm from Sec.~\ref{sec:synthesis} could be adapted to handle trace-dependent $\sgen$ by simply modifying Eq.~\ref{eq:v-recursion}. A detailed study of this adaptation is left for future work.

%BEFORE SHORTENING
%\paragraph{Fairness shields with humans in the loop.}
%In some applications, decisions are made by human experts, and AI-based systems (like classifiers) are deployed to guide the decision-making process~\cite{green2019principles}.
%In these cases, shields may not have the authority to make final decisions.
%But they can serve as a runtime ``fairness filter,'' which would modify and de-bias the original outputs of the decision-maker before passing them on to the human expert.
%This way they can compliment the decision-making process from the fairness standpoint.
%We leave this extension for future work.
%Technically speaking, the recursive synthesis algorithm (in Sec.~\ref{sec:synthesis}) need to be adapted with a new non-deterministic choice (the expert) that determines whether the shield's output is implemented or not.

\paragraph{Fairness shields with humans in the loop.}
In applications where human experts make decisions with AI assistance, shields may not have final decision authority but can act as a runtime "fairness filter" to modify and de-bias the AI's outputs before presenting them to the human expert. 

%BEFORE SHORTENING
%\paragraph{Other future directions.}
%Several other future extensions are possible.
%First, it will be valuable to extend the static and dynamic shields to broader classes of fairness properties, %beyond DoR.
%Second, a comparative study of fine-grained intervention costs will be beneficial, to understand better how to choose a cost model that causes the least amount of utility loss due to shield interventions.
%Finally, it is realistic to expect that the probability distribution over the inputs is not exact and involves uncertainties (e.g., the exact probabilities are replaced by intervals), possibly because they were obtained through statistical estimations.
%We will extend fairness shields so that they provide optimality guarantees against uncertainties in the input distribution.
% Finally, it will be interesting to explore the alternate use of fairness shields with human expert in the loop (see above), where a survey can be set up with human experts to test the effectiveness.

\paragraph{Other future directions.}
Valuable future work includes extending static and dynamic shields to broader classes of fairness properties beyond DoR. Additionally, a comparative study of fine-grained intervention costs would help identify cost models that minimize utility loss. 
Lastly, we plan to extend fairness shields to offer optimality guarantees when input distributions involve uncertainties, e.g., by replacing exact probabilities with intervals. 

% \paragraph{Conclusion.}

% \section{Conclusion}

% \KM{Just putting it here...}
% \subsubsection{Benefits and pitfalls of static shielding.}
% The static approach is computationally cheap, because the shield computation takes place only once.
% On the downside, static shielding either requires restrictive assumptions on the traces (\StaticDP requires traces to be $(T/2)$-balanced), or impose conservative bounds on the group-specific welfare (\StaticBAR imposes the fixed welfare bounds $l,u$).
% We conjecture that this trade-off is unavoidable, because a statically computed shield makes decisions by only looking at most $T$ steps ahead, which is suspected to be not enough to guarantee periodic fairness on arbitrarily long traces.
% Nevertheless, in our experiments, we show that \StaticDP and \StaticBAR often exhibit superior performances even when the assumptions are violated.

% \KM{On the assumption used in dynamic shielding...}

% This assumption is not so intuitive as Assump.~\ref{assump:static-BAR}.
% However, making the condition symmetric for both groups $g\in \aG$ and taking the worst case for $\DP(\tau)$, 
% we arrive to a sufficient condition that $n_g(\tau\tau') \geq \left\lceil  2/\kappa\right\rceil$. 
% This condition parallels the one given for $\StaticBAR$ shields,
% with one key difference: this condition is on $n_g$ accumulated over the trace $\tau\tau'$. 
% In particular, for traces $\tau$ with large enough values of $n_g(\tau)$, this condition is satisfied regardless of $\tau'$.
% We demonstrate in our experiments that this is indeed a common occurrence.

\bibliography{references}

% \section*{Reproducibility Checklist}
% \input{72_reproducibility_checklist}

\ifthenelse{\boolean{includeappendices}}{
    \appendix
    
    \clearpage
    
    \section{Additional Insights}
    \label{sec:reflections}
    In this appendix, we discuss additional insights into the assumptions that we have throughout the paper. 

%%%% FC: I comment this section out, since we are not talking about it in the paper and it's not finished
% \subsection{Empirical DP as an estimator of DP}
% \label{sec:empiricalDP}
% \input{86_empirical_DP_estimator}

% \todo{Derive the expected mean squared error of empirical dp as an estimator on DP. Can this be yours @Konstantin?}

\subsection{Existence of \FinShield}
    The set of feasible solutions of the optimization problem in Eq.~\ref{eq:finite horizon shield} is nonempty for DoR properties, because the fairness-shield that always accepts or always rejects each candidate from each group is a solution that trivially fulfills $\spec(\tau)\leq \kappa$.
    Even nontrivial optimal fairness-shields may exhibit such degenerate behaviors at runtime, when the order of appearances of individuals from the two groups is excessively skewed. 
    Consider the following example for demographic parity (DP).
    Let $\kappa < \frac{1}{T}$.
    Suppose at time $T-1$, all the individuals seen so far were from group $a$ (i.e., $n_a=T-1$ and $n_b=0$).
    If some of the individuals were accepted and the rest rejected, then $0<n_{a1}<n_a$, implying $\kappa<\frac{n_{a1}}{n_a}< 1-\kappa$.
    Now if the $T$-th individual $x$ is from group $b$, $n_b$ becomes $1$, and no matter which action the shield picks, DP will be violated:
    If $x$ is accepted, then $n_{b_1} = \frac{n_{b1}}{n_b}=1$, and if $x$ is rejected, then $n_{b1}=\frac{n_{b1}}{n_b}=0$.
    In both cases, $\DP\left(\left| 
     \frac{n_{a1}(\tau)}{n_a(\tau)} - \frac{n_{b1}(\tau)}{n_b(\tau)}  \right|\right)>\kappa$.
    Therefore, the shield must have made sure that each individual until time $T-1$, all of whom were from group $a$, were either accepted or rejected.
    Luckily, the chances of such skewness of appearance orders is rare in most applications, so that \FinShield as in Def.~\ref{def:finshield} exhibit effective, non-trivial behaviors in most cases as seen from our experiments.

\subsection{Counterexample families for \StaticDP shields being not composable}
\label{sec:counterexamples-static-fair}
We have already shown in Example~\ref{ex:counter-example-periodic-naive-bis} that traces can have zero bias in terms of $\DP$, and when composed have arbitrarily close to 1. 

While the family of counterexamples presented in Example~\ref{ex:counter-example-periodic-naive-bis} is quite degenerate in the sense that acceptance rates are always either 0 or 1, 
we present here another family of counterexamples that is less degenerate.
We write these examples for demographic parity, but the same ideas can be applied to build counterexamples for any DoR property.

Let $T > 0$ and $ 0 < K < T/2$. The family of counterexamples will be parametrized by $(T, K)$. 

    For a pair $(T,K)$ consider 
    traces $\tau_1$, $\tau_2$ such that 
    $(n_{a1}, n_a, n_{b1}, n_b)(\tau_1) = (1, K, 1, T-K)$, 
    and 
    $(n_{a1}, n_a, n_{b1}, n_b)(\tau_2) = 
    (T-K-1, T-K, K-1, K)$.

    In the trace $\tau_1$, exactly one element of each group was accepted, while in the trace $\tau_2$, all but one element of each group were accepted. 
    The values of demographic parity are:
    \begin{equation} \label{eq:dp1}
        \DP(\tau_1)_{T,K} = \left|
        \frac{1}{K} - \frac{1}{T-K}\right| = \frac{T-2K}{(T-K)K}
    \end{equation}

    \begin{equation}\label{eq:dp2}
        \DP(\tau_2)_{T,K} = \left|
        \frac{T-K-1}{T-K} - \frac{K-1}{K}
        \right| = 
        \frac{T-2K}{(T-K)K}
    \end{equation}

    \begin{equation}\label{eq:dp1+2}
        \DP(\tau_1\tau_2)_{T,K} = \left|
        \frac{T-K}{T} - \frac{K}{T}
        \right| = 
        \frac{T-2K}{T}.
    \end{equation}

    These pair of traces are not a counterexample for every pair $(T,K)$.
    However, we can observe that, once fixed $K$, the limit when $T\to \infty$ of Eq.~\ref{eq:dp1} and Eq.~\ref{eq:dp2} is $1/K$, but the limit when $T\to\infty$ of Eq.~\ref{eq:dp1+2} is 1.
    Therefore, for every $\varepsilon$, we can find $K$ large enough such that $1/K < \varepsilon/2$, 
    and then find $T$ large enough such that the corresponding $\DP$ values are close enough to the limit.

    We now build a different family of counterexamples that show that the condition for correctness of \StaticDP shields given in Thm.~\ref{thm:static shield with bounded DP} is as tight as can be.

\begin{theorem}
    For all $\kappa > 0$, there exists 
    $\kappa_1$ and $\kappa_2$
    $\kappa_1 \leq  \kappa \leq \kappa_2$,
    such that for $i\in\{1,2\}$, there exists $t_i$
    and traces $\tau_i, \tau_i'$ that are 
    $\lfloor\frac{t_i-1}{2}\rfloor$-balanced
    such that $\DP(\tau_i)\leq \kappa_i$, $\DP(\tau_i')\leq \kappa_i $,
    and $\DP(\tau_i\tau_i') > \kappa_i$.
\end{theorem}

\begin{proof}
    We prove this theorem by constructing families of counterexamples. 
    For this proof, we use the (slightly abusive) notation that a trace is composed by its four counters, 
    so $\tau = (n_{a}(\tau), n_{a1}(\tau), n_b(\tau), n_{b1}(\tau))$.

    Let $t = 2T+1$ with $T$ even. Consider the traces $\tau_1 = (T+1, T/2 + 1, T, T/2)$ and $\tau_2 = (T, T/2, T+1, T/2)$. Both traces are $T$-balanced. Let's compute demographic parity:
    \begin{equation*}
        \DP(\tau_1) = \frac{T/2 +1}{T+1} - \frac{T/2}{T} = \frac{1}{2(T+1)}
    \end{equation*}
    \begin{equation*}
        \DP(\tau_2) = \frac{T/2}{T} - \frac{T/2}{T+1} = \frac{1}{2(T+1)}
    \end{equation*}
    \begin{equation*}
        \DP(\tau_1\tau_2) = \frac{T+1}{2T+1} - \frac{T}{2T+1} = \frac{1}{2T+1}
    \end{equation*}

    It is clear that $\DP(\tau_1) = \DP(\tau_2) < \DP(\tau_1\tau_2)$.
    Ideally, we would choose $T$ such that $\frac{1}{2(T+1)} = \kappa$, 
    which can be rewritten to $T = \frac{1-2\kappa}{2\kappa}$. 
    However, this may not be an integer. So, given $\kappa$, we take 
    \[
    T_1 = \left\lfloor \frac{1-2\kappa}{2\kappa} \right\rfloor,
    \quad \mbox{and} \quad
    T_2 = \left\lceil \frac{1-2\kappa}{2\kappa} \right\rceil,
    \]
    and define $\kappa_i = \frac{1}{2(T_i+1)}$.
    
    This finishes the construction for an odd $t$. 
    For an even $t$, we show a similar construction. 
    Let $t = 2T$. Consider the traces 
    $\tau_1 = (T+1, 2 ,T - 1, 1)$, $\tau_2 = (T-1, 1, T+1, 1)$. 
    Both traces are $(T-1)$-balanced.
    Let's compute demographic parity:
    \begin{equation*}
        \DP(\tau_1) = \frac{2}{T+1} - \frac{1}{T-1} = \frac{T-3}{T^2-1}
    \end{equation*}
    \begin{equation*}
        \DP(\tau_2) = \frac{1}{T-1} - \frac{1}{T+1} = \frac{2}{T^2-1}
    \end{equation*}
    \begin{equation*}
        \DP(\tau_1\tau_2) = \frac{T+1}{2T+1} - \frac{T}{2T+1} = \frac{1}{2T+1}
    \end{equation*}
    This finishes the proof.   
\end{proof}
The construction proving the theorem is for time horizons $t$ that are $t \equiv 1 (\mod 4)$. Similar constructions can be found for other congruence classes.

\subsection{Existence of \StaticBAR Shields}
\label{sec:condition-on-BAR-shields-existing}
The feasiblity of the condition $N \geq \left\lceil \frac{1}{u-l}\right\rceil$
(Thm.~\ref{thm:acc-rates-existence})
depends on the values of $l$ and $u$ to enforce, as well as the incoming probability distribution. 
In its most simplified form, if we just care about the group membership of any incoming candidate, the distribution of incoming candidates follows a Bernoulli distribution $B(p)$, where $p$ is the probability to receive a candidate of group A. 
After a time horizon $T$, the number of incoming candidates of group A follows a binomial distribution $Bin(T, p)$, and the probability to see at least $N$ candidates of each group is the probability of the binomial being between $N$ and $T-N$, which is 
\[
\sum_{k=N}^{T-N}\binom{T}{k}p^k(1-p)^{T-k}.
\]
In practice, this corresponds to the probability that our shield will encounter a trace where demographic parity with the given bound on acceptance rates can be enforced. 
It is up to the user to evaluate whether this guarantee is enough for a given application.

\subsection{Existence of \Dyn Shields}

The synthesis problem for dynamic shields may not be feasible in general.
This is because there may be traces of length $jT$ that satisfy a certain DP constraint, 
but no shield can guarantee the next trace will satisfy the same constraint. 

\begin{example}\label{ex:buffered}
    Consider $\kappa = 0.1$, $T=100$, and a trace $\tau$ such that $n_a(\tau) = 2$, $n_{a1}(\tau) = 1$, $n_b(\tau) = 98$, and $n_{b1}(\tau) = 49$. The trace $\tau$ satisfies $\DP(\tau) = 0$.
    Now assume we build a shield for the next fragment, and in generating the next trace $\tau'$, only individuals from group $b$ have appeared for the first 99 samples. Let $\tau'_{[1:99]}$ denote this trace, and let $\mathtt{Acc}_b$ denote $\AccRate{b}{\tau\tau'_{[1:99]}}$.
    Then $\DP(\tau\tau'_{[1:99]}) = |1/2 - \mathtt{Acc}_b|$.
    If the last individual of $\tau'$ happens to be from group $a$, the acceptance rate of group $a$ moves from $1/2$ to either $1/3$ (if it gets rejected) or $2/3$ (if it gets accepted). 
    There is no possible value of $\mathtt{Acc}_b$
    that simultaneously guarantees $|1/3 - \mathtt{Acc}_b| \leq \kappa$ and $|2/3 - \mathtt{Acc}_b| \leq \kappa$.
\end{example}

% \todo{Write comment on the condition on Thm~\ref{thm:buff-shield-exists}}

    \section{Detailed Proofs}
    \label{sec:detailed-proofs}
    \subsection{Recursive Shield Synthesis}

\begin{lemma}\label{lem:v-recursion}
    Let $\sgen\in \distrset(\aX)$ be a given joint distribution of sampling individuals and the output of the agent $\sagent$, let $\kappa>0$ be a given threshold for a fairness property $\spec$, and let $T>0$ be a time horizon.
    Let $\sshield^*$ be the shield that minimizes the expected cost after time $T$, i.e.,
\begin{align*}
    \sshield^* \coloneqq \arg\min_{\sshield\in\sShieldFeas^{\sgen,T} } \expe[\cost;\sgen,\sshield,T].
\end{align*}
For a trace $\tau\in (\aX\times\aY)^{\leq T}$, 
let $v(\tau)$ be  the minimum expected cost after a trace $\tau$:

\begin{align*}
    v(\tau) \coloneqq \min\limits_{\pi\in \sShieldFeas^{\sgen,(T-\len{\tau}) \mid\tau}} \expe[\cost\mid \tau;\sgen,\sshield,T-|\tau|].
\end{align*}

Then for $\tau$ with length $|\tau| = T$
\begin{equation}\label{eq:v-basecase-bis}
 v(\tau) = 
        \begin{cases}
            0   & \spec(\tau)\leq\kappa,\\
            \infty  &   \text{otherwise}.
        \end{cases}
\end{equation}
And for $\tau$ with $|\tau| < T$:
\begin{equation}\label{eq:v-recursion-bis}
    \! v(\tau) = \sum_{\cX = (g,r, c)\in\aX} \sgen(\cX)\cdot \min\left\lbrace\begin{matrix}
        v(\tau,(x,y=r)),\\
        v(\tau,(x,y\neq r)) + c
    \end{matrix}\right\rbrace.
\end{equation}

\end{lemma}

\begin{proof}
    Consdier the term to be minimized:
    \begin{align}\label{eq:local5}
        \expe[\cost\mid \tau;\sgen,\sshield,T-|\tau|] = 
        \sum_{\tau'\in (\aX\times\aY)^{T-|\tau|}}
        \cost(\tau')\cdot
        \prob(\tau'\mid\tau;\sgen,\sshield).
    \end{align}
    The sum over traces $\tau'\in (\aX\times\aY)^{T-|\tau|}$
    can be partitioned into a sum over inputs $x\in\aX$
    and traces $\tau''\in (\aX\times\aY)^{T-|\tau|-1}$,
    by taking $\tau' = x\sshield(\tau,x)\tau''$.
    The cost term is then
    \[
    \cost(\tau') = 
    \cost(x\sshield(\tau,x)\tau'') = \cost(\tau'') + \cost(x\sshield(\tau,x)).
    \]
    If $x = (g,r,c)$, then
    \[
    \cost(x\sshield(\tau,x)) = \begin{cases}
        0 \quad \mbox{ if } r = \sshield(\tau,x) \\ 
        c \quad \mbox{ otherwise. }
    \end{cases}
    \]
    The probability term is then:
    \[
    \prob(\tau'\mid\tau;\sgen,\sshield) = 
    \prob(x\sshield(\tau,x) \mid \tau;\sgen,\sshield)\cdot
    \prob(\tau''\mid \tau x\sshield(\tau,x);\sgen,\sshield)
    \]
    The value in Eq.~\ref{eq:local5} can be written as 
    \begin{equation}\label{eq:local6}
    \sum_{\tau'\in (\aX\times\aY)^{T-|\tau|}}
        \cost(\tau')\cdot
        \prob(\tau'\mid\tau;\sgen,\sshield) = 
        A + B,
    \end{equation}
    where 
    \begin{multline}
    \label{eq:local7}
        A = \sum_{x\in\aX
        }\sum_{\tau''\in (\aX\times\aY)^{T-|\tau|-1}}
        \cost(\tau'') \cdot 
        \\
        \cdot \prob(x\sshield(\tau,x) \mid \tau;\sgen,\sshield)\cdot
        \prob(\tau''\mid \tau x\sshield(\tau,x);\sgen,\sshield),
    \end{multline}
    and 
    \begin{multline}
    \label{eq:local8}
    B =\sum_{x\in\aX
        }\sum_{\tau''\in (\aX\times\aY)^{T-|\tau|-1}}
    \cost(x\sshield(\tau,x)) 
    \\ 
    \prob(x\sshield(\tau,x) \mid \tau;\sgen,\sshield)\cdot
    \prob(\tau''\mid \tau x\sshield(\tau,x);\sgen,\sshield).
    \end{multline}
     Note that the term $\prob(x\sshield(\tau,x) \mid \tau;\sgen,\sshield)$ appears several times. 
    This is the probability of getting a trace $x\sshield(\tau,x)$ after having seen a trace $\tau$. 
    This is, by definition $\prob(x\sshield(\tau,x) \mid \tau;\sgen,\sshield) = \sgen(x)$.
    
    Since $\sgen(x)$
    does not depend on $\tau''$, 
    the sum in $A$ can be rearranged as 
    \begin{multline}
         A = \sum_{x\in\aX
        }
        \sgen(x)\cdot
        \sum_{\tau''\in (\aX\times\aY)^{T-|\tau|-1}}
        \cost(\tau'')
        \cdot \\
        \prob(\tau''\mid \tau x\sshield(\tau,x);\sgen,\sshield),
    \end{multline}
    and therefore
    \begin{multline}
         A = \sum_{x\in\aX
        }\sgen(x)
        \cdot
        \expe[\cost\mid \tau x\sshield(\tau,x);\sgen,\sshield,T-|\tau|-1].
    \end{multline}
    The term $B$ can be similarly rearranged, taking into consideration that in this case $\cost(x\sshield(\tau,x))$ is also independent of $\tau''$:
    \begin{multline}
    B =\sum_{x\in\aX
        }
    \cost(x\sshield(\tau,x))\cdot
    \sgen(x) \\
    \sum_{\tau''\in (\aX\times\aY)^{T-|\tau|-1}}
    \prob(\tau''\mid \tau x\sshield(\tau,x);\sgen,\sshield).
    \end{multline}
    The hanging term is the sum of probabilities, so by definition adds up to 1:
    \[
     \sum_{\tau''\in (\aX\times\aY)^{T-|\tau|-1}}
    \prob(\tau''\mid \tau x\sshield(\tau,x);\sgen,\sshield) =1.
    \]
    Therefore 
    \begin{equation}
        B =\sum_{x\in\aX
        }\sgen(x)\cdot
    \cost(x\sshield(\tau,x))
    \end{equation}

    Putting $A$ and $B$ together we get:
    \begin{multline}
        \expe[\cost\mid \tau;\sgen,\sshield,T-|\tau|]  
        = \\ \hspace{-6em}
        \sum_{x\in\aX} \sgen(x)\cdot\Big(
        \cost(x\sshield(\tau,x)) + \\
        \expe[\cost\mid \tau x\sshield(\tau,x);\sgen,\sshield,T-|\tau|-1]
        \Big)
    \end{multline}
    This partitions the value of $\expe[\cost\mid \tau;\sgen,\sshield,T-|\tau|]$ into a sum of cost of current decision ($\cost(x\sshield(\tau,x))$)
    and expected cost in the rest of the trace. 
    For every $x$, the optimal value of the shield $\pi(\tau,x)$ is the one that minimizes 
    \begin{equation*}
        \cost(x\sshield(\tau,x)) + \\
        \expe[\cost\mid \tau x\sshield(\tau,x);\sgen,\sshield,T-|\tau|-1].
    \end{equation*}
    This is precisely, the recursive property that we want to prove.
\end{proof}

\noindent\textbf{Theorem \ref{thm:bounded-horizon shield synthesis-bis}}. 
\emph{The bounded-horizon shield-synthesis problem 
 can be solved in $\mathcal{O}(T^p\cdot |\aX|)$-time and $\mathcal{O}(T^p\cdot |\aX|)$-space.}
 \begin{proof}
    In Sec.~\ref{sec:synthesis} we describe a dynamic programming approach to synthesize the shield by recursively computing $v(\tau)$ for all possible traces $\tau\in(\aX\times\aY)^{\leq T}$.
    As explained in Sec.~\ref{sec:synthesis:efficient}, 
    these computations do not depend directly on $\tau$, 
    but rather on the statistic $\mu$, 
    that depends on $p$ counters, taking values in the interval $[0,T]$. 
    We need to build a table with the shield values for every pair of counter values and input.
    Therefore, the table occupies a space $\mathcal O(T^p\cdot|\aX|)$. 
    Every element of the table has to be computed only once, and it is done as a sum over all elements of $x$, thus the cost in time is $\mathcal O(T^p\cdot |\aX|)$.
 \end{proof}

\subsection{Static Shielding}

\noindent\textbf{Theorem \ref{thm:static shield with bounded DP}} (Conditional correctness of \StaticDP shields)
\emph{
    Let $\spec$ be a DoR fairness property.
    For a given \StaticDP shield $\sshield$,
    let $\tau = \tau_1\ldots\tau_m\in \Trfeas^{mT}$ be a trace with $\len{\tau_i}=T$ for each $i\leq m$.
    If for every $i\leq m$, $\tau_i$ is $(T/2)$-balanced w.r.t.\ $\den^a$ and $\den^b$,
    then $\spec(\tau)\leq \kappa$.}

\begin{proof}
    % \todo{Write complete proof (commented out)}
    For each $i\leq m$, if $\den^a(\tau_i) = \den^b(\tau_i) = T/2$, 
    then $\welfare{g}(\tau_i) = (2/T)\numer^g(\tau_i)$,
    and therefore $\spec(\tau_i) = (2/T)|\numer^a(\tau_i) - \numer^b(\tau_i)|$.
    If $\spec(\tau_i)\leq \kappa$,
    $|\numer^a(\tau_i) - \numer^b(\tau_i)|\leq \kappa T/2$.

    Applying the triangular inequality followed by the previous result for all $i$, we get
    \begin{align*}
    \left|\sum_i (\numer^a(\tau_i)-\numer^b(\tau_i))\right|
    & \leq 
    \sum_i\left|\numer^a(\tau_i) - \numer^b(\tau_i)\right|\\    
    & \leq 
    m\kappa T/2.
    \end{align*}
 \end{proof}

\noindent\textbf{Lemma \ref{prop:acc-rates}.}
\emph{Let $(l,u)$ be given welfare bounds, and $\welfare{g}(\cdot)\equiv \numer^g(\cdot)/\den^g(\cdot)$ for additive $\numer^g,\den^g$.
    For every trace $\tau = \tau_1\dots \tau_m$,
    if for each $i\leq m$, 
    $\welfare{g}(\tau_i) \in [l,u]$, 
    then 
    $\welfare{g}(\tau) \in [l,u]$.}

To prove Lemma~\ref{prop:acc-rates}, we first need to prove the following auxiliary result.

\begin{lemma}\label{lem:min-frac-max-bis}
    Let $a_1, \dots, a_m$, $b_1, \dots, b_m$ be positive real numbers. 
    Then
    \begin{equation}\label{eq:min-frac-max-lemma}
        \min_{i\in\{1\dots m\}} \frac{a_i}{b_i} \leq 
        \frac{\sum_{i=1}^m a_i}{\sum_{i=1}^m b_i} \leq 
        \max_{i\in\{1\dots m\}} \frac{a_i}{b_i}.
    \end{equation} 
\end{lemma}
\begin{proof}
    This is an extension of the following known inequality: 
    given positive numbers
    $w, x, y, z$, if $w/x < y/z$, then
    $\frac{w}{x} \leq \frac{w+y}{x + z} \leq \frac{y}{z}$.
    We can restate it as:
    \begin{equation} \label{eq:min-frac-max1-bis}
        \min\left(
        \frac{w}{x}, \frac{y}{z} 
        \right)
        \leq \frac{w + y}{x + z} 
        \leq 
        \max\left(
        \frac{w}{x}, \frac{y}{z} 
        \right).
    \end{equation}

    We prove this result by induction on $m$. The base case for $m=1$ is trivial.
    
    For a general $m$, we start applying inequality~\ref{eq:min-frac-max1-bis} with
    $w = \sum_{i=1}^{m-1} a_i$, 
    $x = \sum_{i=1}^{m-1} b_i$,
    $y = a_m$, and $z = b_m$, 
    to obtain:
    \begin{equation*}
        \frac{\sum_{i=1}^m a_i}{\sum_{i=1}^m b_i} \leq
        \max\left(
        \frac{\sum_{i=1}^{m-1} a_i}{\sum_{i=1}^{m-1} b_i}, 
        \frac{a_m}{b_m}
        \right).
    \end{equation*}
    Applying the induction hypothesis we have that 
    \begin{equation}
        \frac{\sum_{i=1}^{m-1} a_i}{\sum_{i=1}^{m-1} b_i} \leq 
        \max_{i\in \{1\dots m-1\}} \frac{a_i}{b_i},
    \end{equation}
    and therefore:
    \begin{equation*}
        \frac{\sum_{i=1}^m a_i}{\sum_{i=1}^m b_i} \leq
        \max\left(
        \max_{i\in \{1\dots m-1\}} \frac{a_i}{b_i}, 
        \frac{a_m}{b_m}
        \right) = \max_{i\in \{1\dots m\}} \frac{a_i}{b_i}.
    \end{equation*}
    This proves the right-side inequality of Eq.~\ref{eq:min-frac-max-lemma}. 
    The left-side is analogous.
\end{proof}

\begin{proof}[Proof (Of Lemma~\ref{prop:acc-rates})]
    Let $n^a_i = \den^a(\tau_i)$, $n^{a1}_i = \numer^a(\tau_i)$, 
    $n^b_i = \den^b(\tau_i)$, and $n^{b1}_i = \den^b(\tau_i)$.
    
    Applying Lemma~\ref{lem:min-frac-max-bis}, we have 
    for all $g\in\aG$
    that 
    \begin{equation}\label{eq:min-frac-max-lemma-bis}
        \min_{i\in\{1\dots n\}} \frac{n^{g1}_i}{n^g_i} \leq 
        \frac{\sum_{i=1}^n n^{g1}_i}{\sum_{i=1}^n n^{g}_i} \leq 
        \max_{i\in\{1\dots n\}} \frac{n^{g1}_i}{n^g_i}.
    \end{equation} 

    And we also know that all welfare values are bounded by $l$ and $u$. That is, for all $i\in\{1\dots n\}$ and all $g\in\aG$

    \begin{equation} \label{eq:beta-alpha-bis}
        l \leq  \frac{n^{g1}_i}{n^{g}_i} \leq u
    \end{equation}

    In particular, Eq.~\ref{eq:beta-alpha-bis} applies to the maximum and minimum welfare values. 
    This, together with Eq.~\ref{eq:min-frac-max-lemma-bis}
    finishes the proof.
    \end{proof}

\begin{lemma}
    \label{thm:acc-rates-existence}
    Let $\spec$ be a DoR property with $\spec(\tau) = |\welfare{a}(\tau) - \welfare{b}(\tau)|$, and $\welfare{g}(\tau) = \numer^g(\tau)/\den^g(\tau)$.
    Let $0\leq l < u\leq 1$ be a pair of welfare bounds. The set of shields
    \begin{multline*}
         \sShieldFeasBounded^{T, N} \coloneqq \set{\sshield\in\sShield \mid \forall \tau \in \Trfeas^t_{\sgen, \sshield} \cap \Trbalance_N^T \;.\; \\ 
        \forall g\in \set{a,b}\;.\;
        l\leq \welfare{g}(\tau) \leq u }
    \end{multline*}
is not empty for $N \geq \left\lceil \frac{1}{u-l}\right\rceil$.
\end{lemma}

\begin{proof}[Proof]
    For a shield to exist that can enforce bounds $[l, u]$ on the welfare, there must exist, for every value of $\den^g(\tau)$, 
    at least one way of deciding for increasing or not $\numer^g(\tau)$ that maintains the welfare in the desired bounds. 
    Since we do not know \textit{a priori} 
    the value of $\den^g(\tau)$, this decision must be incremental, 
    and be such that the welfare is maintained for any value of $\den^g(\tau)$. 

    To express this, there needs to exist a sequence $(x_n)\subseteq\mathbb N$ for all $n\geq N$ such that
    \begin{equation}\label{eq:xn}
        l \leq \frac{x_n}{n} \leq u, \quad \mbox{and} 
        \quad x_{n+1} - x_n \in \{0,1\}.
    \end{equation}
    Given $l$, and $u$, if $\den^g(\tau)$ is at least $N$ for a given group $g$, the shield can force $\numer^g(\tau)$ to be exactly $x_n$ to ensure the bound on welfare is met. 

    The condition in Eq.~\ref{eq:xn} can be 
    reformulated as $ ln \leq x_n \leq  un$,
    and since $x_n$ needs to be an integer, 
    we can tighten it to 
    \begin{equation}\label{eqn:local2}
        \left\lceil l n \right\rceil \leq x_n \leq 
        \left\lfloor u n \right\rfloor.
    \end{equation}
    One option is to try $x_n = \left\lceil l n \right\rceil$.
    We have to prove that this choice satisfies two conditions: 
    (i) $x_{n+1} - x_n \in \{0,1\}$, and
    (ii) Equation~\ref{eqn:local2}.

    \begin{enumerate}
        \item[(i)] 
        This is true for any sequence $x_n$ built as the integer part of $nl$, where $l\in[0,1]$. 
        For any number $x$, it is known that 
        $x = \lceil x\rceil - \{x\}$, where $0 \leq \{x\} < 1$. 
        Applying this inequality twice, we get 
        $x_{n+1} - x_n = \lceil l(n+1)\rceil - \lceil l n\rceil
         < l (n+1) - \lceil l n\rceil \leq l(n+1)- l n = 1+ l < 2$. 
        Since $\lceil l (n+1)\rceil - \lceil l n\rceil$ is an integer strictly smaller than 2, it is smaller or equal than 1. It is also clearly non-negative.

        \item[(ii)]
        By construction, $l n \leq \lceil l n\rceil$. 
        Now we have to see that $\lceil l n\rceil \leq u n$. 
        If $\lceil l n\rceil = l n$, 
        then for any $n\geq 1$, $x_n \leq u n$ on account of $l < u$. 
        If $\lceil l n\rceil = l n +1$, 
        we need $l n + 1 \leq u n$, 
        which is equivalent to $n \geq \frac{1}{u-l}$. 
        Since $n$ needs to be an integer, 
        selecting $N = \left\lceil \frac{1}{u-l}\right\rceil$ ensures this condition is satisfied. 
    \end{enumerate}
\end{proof}

\noindent\textbf{Theorem \ref{thm:correctness-staticBW}} (Conditional correctness of \StaticBAR shields)

\emph{Let $\spec$ be a DoR fairness property.
    Let $l,u$ be welfare bounds such that $u-l\leq \kappa$.
    For a given \StaticBAR shield $\sshield$, 
    let $\tau = \tau_1\ldots\tau_m\in \Trfeas^{mT}$ be a trace with $\len{\tau_i}=T$ for each $i\leq m$.
    If Assump.~\ref{assump:static-BAR} holds, then $\spec(\tau)\leq \kappa$.}

\begin{proof}
    This is a direct consequence of Lemmas~\ref{prop:acc-rates}~and~\ref{thm:acc-rates-existence}.
    If Assump.~\ref{assump:static-BAR} holds, Lemma~\ref{thm:acc-rates-existence} ensures that the set of shields is non-empty. 
    Furthermore, any such shield satisfies the fairness condition $\spec(\tau)\leq \kappa$ for any trace in $\tau\in\Trfeas^{mT}$ 
    by Lemma~\ref{prop:acc-rates}.
\end{proof}

\subsection{Dynamic Shielding}

\begin{lemma}
    \label{lem:buff-shield-exists}
    Let $\spec$ be a DoR fairness property 
    with $\spec(\tau) = |\welfare{a}(\tau) - \welfare{b}(\tau)|$, and $\welfare{g}(\tau) = \numer^g(\tau)/\den^g(\tau)$.
    Let
    $\tau_1$ be a trace and $\kappa \geq 0$. 
    There exists a shield $\sshield\in\sShield$ such that
    every trace $\tau_2\in \Trfeas^T_{\sgen, \sshield} \cap S$
    satisfies $\DP(\tau_1\tau_2)\leq \kappa$, where 
    \begin{multline*}
    S = \set{\tau_2\in (\aX\times\aY)^T\::\: \\  \frac{1}{\den^a(\tau_1\tau_2)} + \frac{1}{\den^b(\tau_1\tau_2)} \leq \kappa + \spec(\tau_1)}
    \end{multline*}
\end{lemma}

% \noindent\textbf{Theorem \ref{thm:buff-shield-exists}.}
% \emph{
%     Let $\tau_1$ be a trace and $\kappa \geq 0$. 
%     There exists a shield $\sshield\in\sShield$ such that
%     every trace $\tau_2\in \Trfeas^T_{\sagent, \pi} \cap S$
%     satisfies $\DP(\tau_1\tau_2)\leq \kappa$, where 
%     \[
%     S = \left\{\tau_1\in (\aX\times\aY)^T\::\: \frac{1}{n_{a}(\tau_1\tau_2)} + \frac{1}{n_b(\tau_1\tau_2)} \leq \kappa + \DP(\tau_1)\right\}
%     \]
% }

\begin{proof}
    The proof of this result is analogous to that of Lemma~\ref{thm:acc-rates-existence},
    with a slightly more convoluted argument.

    Let $n^1_a = \den^a(\tau_1)$, 
    $n^1_{a1} = \numer^a(\tau_1)$, 
    $n^1_b = \den^b(\tau_1)$, 
    and $n^1_{b1} = \numer^{b}(\tau_1)$.
    Without loss of generality, we can assume that 
    $n^1_{a1}/n^1_a - n^1_{b1}/n^1_b \geq 0$. 
    The alternative case is analogous.

    For a shield to exist that can enforce $\spec(\tau_1\tau_2) \leq \kappa$ there must exist, for every value of $\den^g(\tau_1\tau_2)$ (in demographic parity, the amount of individuals of a group),
    at least one way of deciding acceptance and rejection (value of $\numer^g(\tau_1\tau_2)$) 
    that maintains the fairness property in the target bound.
    Since we do not know \textit{a priori} how many individuals of each group will appear, this decision must be incremental, 
    and be such that the fairness property is maintained for any number of individuals. 

    To express this, a new trace $\tau_2$ 
    is enforceable if $n_a = \den^{a}(\tau_2) \geq N_a$ and 
    $n_b = \den^b(\tau_2)\geq N_b$ if there exist two sequences $(x_{n_a}), (y_{n_b})\subseteq\mathbb N$
    such that for all $n^a\geq N_a$ and $n^b \geq N_b$
    \begin{equation}\label{eq:local3}
        \left|
        \frac{n^1_{a1} + x_{n_a}}{n^1_a + n_a} -
        \frac{n^1_{b1} + y_{n_b}}{n^1_b + n_b}
        \right| \leq \kappa
    \end{equation}
    and for both sequences $x_{n_a+1} - x_{n_a} \in \{0,1\}$ and 
    $y_{n_a+1} - y_{n_a} \in \{0,1\}$.

    With the spirit of maintaining the welfare bounds as a proxy to maintaining fairness, we try
    \[
    x_{n_a} = \left\lfloor \frac{n^1_{a1}}{n^1_a}n_a\right\rfloor, \quad \mbox{and} \quad
    y_{n_a} = \left\lceil \frac{n^1_{b1}}{n^1_b}n_b\right\rceil.
    \]
    Using the same argument as in the proof of Theorem~\ref{thm:acc-rates-existence}, point (i), the conditions on $x_{n_a}$ and $y_{n_b}$ incrementing by 0 or 1 are met by the fact that $n^1_{a1}\leq n^1_a$ 
    and  $n^1_{b1}\leq n^1_b$.

    By definition of the floor function and ceiling functions
    \begin{align*}
        \frac{n^1_{a1} + x_{n_a}}{n^1_a + n_a} \leq 
        \frac{n^1_{a1} + \frac{n^1_{a1}}{n^1_a}n_a }{n^1_a + n_a} = 
        \frac{n^1_{a1}}{n^1_a}, \\
        \frac{n^1_{b1} + y_{n_b}}{n^1_b + n_b} \geq 
        \frac{n^1_{b1} + \frac{n^1_{b1}}{n^1_b}n_b }{n^1_b + n_b} = 
        \frac{n^1_{b1}}{n^1_b}.
    \end{align*}
    Therefore
    \begin{equation}
        \frac{n^1_{a1} + x_{n_a}}{n^1_a + n_a} -
        \frac{n^1_{b1} + y_{n_b}}{n^1_b + n_b}
        \leq 
        \frac{n^1_{a1}}{n^1_a} -
        \frac{n^1_{b1}}{n^1_b} = \spec(\tau_1) \leq \kappa
    \end{equation}
    To prove Eq.~\ref{eq:local3}, we still have to prove that
    \begin{equation}\label{eq:local4}
        \frac{n^1_{a1} + x_{n_a}}{n^1_a + n_a} -
        \frac{n^1_{b1} + y_{n_b}}{n^1_b + n_b}
        \geq -\kappa.    
    \end{equation}
    
    By the definition of the floor function
    \[
    \frac{n^1_{a1} + x_{n_a}}{n^1_a + n_a} \geq 
    \frac{n^1_{a1} + \frac{n^1_{a1}}{n^1_a}n_a  - 1}{n^1_a + n_a} =
    \frac{n^1_{a1}}{n^1_a} - \frac{1}{n^1_a + n_a},
    \]
    and by the definition of the ceiling function
    \[
    \frac{n^1_{b1} + y_{n_b}}{n^1_b + n_b} \leq 
    \frac{n^1_{b1} + \frac{n^1_{b1}}{n^1_b}n_b +1}{n^1_b + n_b} = 
    \frac{n^1_{b1}}{n^1_b} + \frac{1}{n^1_b + n_b}.
    \]
    Putting the previous two inequalities together, we have
    \[
    \frac{n^1_{a1} + x_{n_a}}{n^1_a + n_a} -
    \frac{n^1_{b1} + y_{n_b}}{n^1_b + n_b} \geq \spec(\tau_1) - 
    \left(
    \frac{1}{n^1_a + n_a} + 
    \frac{1}{n^1_b + n_b}
    \right).
    \]
    To ensure that Equation~\ref{eq:local4} holds, it is sufficient to 
    ensure that
    \begin{equation*}
        \spec(\tau_1) - 
    \left(
    \frac{1}{n^1_a + n_a} + 
    \frac{1}{n^1_b + n_b}
    \right) \geq -\kappa.
    \end{equation*}
    Rewriting the previous inequality we arrive to
    \begin{equation}
         \left(
    \frac{1}{n^1_a + n_a} + 
    \frac{1}{n^1_b + n_b}
    \right) \leq \kappa + \spec(\tau_1),
    \end{equation}
    which is the condition defining the set $S$.
    Therefore the proposed sequences $(x_{n_a})$ and $(y_{n_b})$ satisfy Equation~\ref{eq:local3} for traces in $S$.
\end{proof}

\noindent\textbf{Theorem \ref{thm:correctness-dyn-shields}} (Conditional correctness of \Dyn shields)
\emph{Let $\spec$ be a DoR fairness property.
    Let $\sshield$ be a  \Dyn shield that uses $\sShieldFeasDyn(\cdot)^{\sgen,T}$ as the set of available shields.
    Let $\tau =  \tau_1\ldots\tau_m\in \Trfeas^{mT}_{\sgen,\sshield}$ be a trace with $\len{\tau_i}=T$ for each $i\leq m$.
    Suppose for every $i\leq m$, $\tau_1\ldots\tau_i$ fulfills Assump.~\ref{assump:dynamic}.
    Then $\DP(\tau)\leq \kappa$.}
\begin{proof}
    This is a direct consequence of Lemma~\ref{lem:buff-shield-exists}.  
    If Assump.~\ref{assump:dynamic} holds, Lemma~\ref{lem:buff-shield-exists} ensures that the set of shields is non-empty. 
    By construction, any such shield satisfies $\spec(\tau)\leq \kappa$ for any trace in $\tau\in\Trfeas^{mT}_{\sgen,\sshield}$.
\end{proof}

% \subsection{Empirical DP as an estimator of DP}
% \todo{Move technical details of Sec. ~\ref{sec:empiricalDP} here.}

    % \section{Approximating the Input Distribution}
    % \label{sec:cost_alignment}
    % \input{76_cost_alignment}
    
    \section{Experimental Setup}
    \label{sec:experimental_setup_appendix}
    In this appendix, we describe our experimental setup in detail. 

%%%%%%%%%%%%%%%%%% table placed here for compilation, belongs to sec. experimental setup
\begin{table*}[t]
    \centering
    \begin{tabular}{l c c c c c c c c c c}
    \toprule
 Dataset & Task & SensAttr & Instances & Num. Feat. & Cat. Feat. & $\% y_0$ & $\% y_1$ & $\% g_a$ & $\% g_b$ & Statistical Parity \\ 
     \midrule
 \texttt{Adult} & income & race & 43131 & 5 & 7 & 75.10 & 24.90 & 9.80 & 90.20 & 0.14 \\ 
 \texttt{Adult} & income & gender & 45222 & 5 & 7 & 75.22 & 24.78 & 32.50 & 67.50 & 0.20 \\ 
 \texttt{Bank} & credit & age & 41188 & 9 & 10 & 88.73 & 11.27 & 2.59 & 97.41 & 0.13 \\ 
 \texttt{\textsc{Compas}} & recidivism & gender & 6172 & 5 & 4 & 54.49 & 45.51 & 19.04 & 80.96 & 0.13 \\ 
 \texttt{\textsc{Compas}} & recidivism & race & 6172 & 5 & 4 & 54.49 & 45.51 & 65.93 & 34.07 & 0.10 \\ 
 \texttt{German} & credit & gender & 1000 & 6 & 13 & 30.00 & 70.00 & 31.00 & 69.00 & 0.07 \\ 
 \texttt{German} & credit & age & 1000 & 6 & 13 & 30.00 & 70.00 & 19.00 & 81.00 & 0.15 \\ 
\bottomrule
    \end{tabular}
    \caption{Datasets characteristics }
    \label{tab:datasets-info}
\end{table*}

\subsection{Computing infrastructure}
All experiments were performed with a workstation with AMD Ryzen 9 5900x CPU, Nvidia GeForce RTX 3070Ti GPU, 32GB of RAM, running Ubuntu 20.04.
The code to reproduce our experiments is included as part of the supplementary material.

\subsection{Datasets}
We used four tabular datasets in our experiments, all common benchmarks in the fairness community: Adult~\cite{misc_adult_2}, COMPAS~\cite{compas},
German Credit~\cite{dua2017uci} and Bank Marketing~\cite{misc_bank_marketing_222}.
Details on the task, sensitive attributes, size of the dataset, number of numerical and categorical features, as well as existing bias can be found in Table~\ref{tab:datasets-info}.

\subsection{Training ML classifiers}
To train our ML models, we adapted the implementation provided by~\cite{han2024ffb}, using the same neural network, train-test splits, and most training hyper-parameters set as default in their implementation, tuning only the hyper-parameters related to fairness.

We use fixed architecture multi-layer perception (MLP) with three hidden layers with sizes $512$, $256$, and $64$ in all our experiments.
In each case, the model is trained for 150 epochs with batches of 1024 instances, with the exception of the $\texttt{German}$ dataset, which we trained with batches of 128, as the dataset has only 1000 instances.
We use the Adam optimizer~\cite{adam}, with a learning rate of 0.01.

\subsection{Learning Algorithms}
To train our classifiers, we used the following methods from the in-processing fairness literature:

\begin{itemize}
    \item Differential Demographic Parity (DiffDP) is a gap regularization method for demographic parity. DiffDP introduces a term in the loss function that penalizes differences in the prediction rates between different demographic groups.
    \item The Hilbert-Schmidt Independence Criterion (HSIC) is a statistical test used to measure the independence of two random variables. Adding an HSIC term measuring the independence between prediction accuracy and sensitive attributes to the loss has been used as a fair learning method~\cite{hsic}.
    \item Learning adversarially fair and transferable representations (LAFTR) is a method proposed by~\cite{madras2018learning}, where the classifier learns an intermediate representation of the data that minimizes classification error while simultaneously minimizing the ability of an adversary to predict sensitive features from the representation.
    \item Prejudice Remover (PR)~\cite{kamishima2012fairness} adds a term to the loss that penalizes mutual information between the prediction accuracy and the sensitive attribute.
\end{itemize}

As a baseline, we trained a fifth classifier simply minimizing empirical risk (ERM).

\subsubsection{Hyperparameter tuning.}

Each of the in-processing fairness algorithms depends on the value of certain parameters that indicate the trade-off in the loss function between prediction accuracy and fairness.
For each training algorithm, we manually fine-tuned the parameters to obtain a good performance with the same parameter values across all benchmarks. 
Unfortunately, the parameters of different algorithms have different interpretations and characteristic dimensions, so comparing them is not informative. We detail the ones we used in our experiments, and their meaning.
\begin{itemize}
    \item For DiffDP, a parameter $\lambda$ controls the contribution of the regularization term in the loss. 
    We tried a range of $\lambda\in[0.5, 10]$.
    We use $\lambda = 1$.
    \item In HSIC, a parameter $\lambda$ controls the importance of the HSIC term in the loss function. We tried a range $\lambda\in [10, 500]$. We use $\lambda = 100$.
    \item In LAFTR, the loss is composed of three terms: 
    one that penalizes reconstruction error ($L_x$), 
    one that penalizes prediction error ($L_y$), 
    and one that penalizes the adversary's error when trying to obtain information about sensitive features from the representation ($L_z$).
    Three parameters $A_x, A_y, A_z$ control the weights of each term in the loss. We use $A_x=8$, $A_y = 4 $, $A_z = 2.1$.
    We tried a range or $[1,10]$ for each parameter.
    \item For PR, a parameter $\lambda$ controls the weight of the loss term that penalizes mutual information between prediction accuracy and the sensitive attribute. 
    We tried a range of $\lambda\in [0.01,0.5]$.
    We use $\lambda = 0.06$.
\end{itemize}
In Tables~\ref{tab:ml-performance-adult},~\ref{tab:ml-performance-bank_marketing},~\ref{tab:ml-performance-compas},~\ref{tab:ml-performance-german} we show the metrics of each trained model on each dataset. 
For each case, we present accuracy ($\mathtt{acc}$), average precision ($\mathtt{ap}$), area under the curve ($\mathtt{auc}$), and the $F1$ score ($\mathtt{f1}$) as performance metrics, 
while demographic parity ($\DP$) and equal opportunity ($\EO$) are presented as fairness metrics.
The numbers are presented as percentages.
In each column, the best performer is marked in boldface.

\begin{table}[t]
    \centering
    \begin{tabular}{l c c c c c c }
    \toprule
   & $\mathtt{acc}$ & $\mathtt{ap}$ & $\mathtt{auc}$ & $\mathtt{f1}$ & $\DP$ & $\EO$ \\ 
     \midrule
 $\mathtt{ERM}$ & 91.3 & 62.0 & 94.2 & \textbf{ 56.8} & 10.5 & \textbf{ 3.7} \\ 
 $\mathtt{DiffDP}$ & 90.4 & 57.3 & 93.1 & 40.3 & 3.4 & 26.2 \\ 
 $\mathtt{HSIC}$ & \textbf{ 91.4} & \textbf{ 62.5} & \textbf{ 94.3} & 56.7 & 7.0 & 6.6 \\ 
 $\mathtt{LAFTR}$ & 90.5 & 60.0 & 93.6 & 39.8 & 6.0 & 4.4 \\ 
 $\mathtt{PR}$ & 90.7 & 58.7 & 93.4 & 49.4 & \textbf{ 3.3} & 34.2 \\ 
\bottomrule
    \end{tabular}
    \caption{Performance of the ML models. Dataset: \texttt{Bank} }
    \label{tab:ml-performance-bank_marketing}
\end{table}

\begin{table*}[t]
    \centering
    \begin{tabular}{l c c c c c c | c c c c c c }
    \toprule
 & \multicolumn{6}{c|}{race} & \multicolumn{6}{c}{gender}\\ \midrule
   & $\mathtt{acc}$  & $\mathtt{ap}$  & $\mathtt{auc}$  & $\mathtt{f1}$  & $\DP$  & $\EO$  & $\mathtt{acc}$  & $\mathtt{ap}$  & $\mathtt{auc}$  & $\mathtt{f1}$  & $\DP$  & $\EO$ \\ 
\midrule 
$\mathtt{ERM}$  & \textbf{85.2}  & \textbf{78.8}  & \textbf{91.2}  & \textbf{66.4}  & 10.9  & 6.1  & \textbf{85.2}  & \textbf{78.6}  & \textbf{91.1}  & \textbf{65.7}  & 16.5  & 9.6 \\ 
$\mathtt{DiffDP}$  & 84.4  & 76.3  & 89.5  & 62.3  & \textbf{5.4}  & 4.0  & 82.5  & 71.4  & 87.2  & 53.9  & 0.2  & 32.8 \\ 
$\mathtt{HSIC}$  & 84.9  & 78.6  & 91.1  & 64.4  & 8.7  & \textbf{3.1}  & 83.2  & 72.7  & 87.3  & 56.9  & 1.8  & 28.0 \\ 
$\mathtt{LAFTR}$  & 84.4  & 77.6  & 90.9  & 62.0  & 9.5  & 8.2  & 85.1  & 78.6  & 91.0  & 65.1  & 14.3  & \textbf{1.5} \\ 
$\mathtt{PR}$  & 84.2  & 75.7  & 89.0  & 61.4  & 5.6  & 4.0  & 82.5  & 70.6  & 86.5  & 53.4  & \textbf{0.1}  & 33.1 \\ 
\bottomrule
    \end{tabular}
    \caption{Performance of the ML models. Dataset: \texttt{Adult} }
    \label{tab:ml-performance-adult}
\end{table*}

\begin{table*}[t]
    \centering
    \begin{tabular}{l c c c c c c | c c c c c c }
    \toprule
 & \multicolumn{6}{c|}{gender} & \multicolumn{6}{c}{race}\\ \midrule
   & $\mathtt{acc}$  & $\mathtt{ap}$  & $\mathtt{auc}$  & $\mathtt{f1}$  & $\DP$  & $\EO$  & $\mathtt{acc}$  & $\mathtt{ap}$  & $\mathtt{auc}$  & $\mathtt{f1}$  & $\DP$  & $\EO$ \\ 
\midrule 
$\mathtt{ERM}$  & \textbf{65.3}  & 62.8  & 68.6  & 58.6  & 16.5  & 18.1  & 64.8  & 63.3  & 69.1  & 59.8  & 14.3  & 15.5 \\ 
$\mathtt{DiffDP}$  & 63.1  & 63.4  & 68.7  & 54.7  & 12.1  & 11.0  & 64.6  & 61.9  & 68.1  & 57.8  & 9.1  & 14.8 \\ 
$\mathtt{HSIC}$  & 64.3  & 63.0  & 68.7  & 56.2  & 14.8  & 9.8  & 64.4  & 61.9  & 68.3  & 56.9  & \textbf{8.2}  & \textbf{11.3} \\ 
$\mathtt{LAFTR}$  & 65.2  & \textbf{64.1}  & \textbf{69.9}  & \textbf{59.8}  & 17.2  & 14.8  & \textbf{65.1}  & \textbf{63.7}  & \textbf{70.3}  & \textbf{59.9}  & 13.2  & 18.2 \\ 
$\mathtt{PR}$  & 63.0  & 63.4  & 68.7  & 54.9  & \textbf{12.0}  & \textbf{8.0}  & 64.2  & 62.2  & 68.3  & 56.8  & 8.9  & 13.1 \\ 
\bottomrule
    \end{tabular}
    \caption{Performance of the ML models. Dataset: \texttt{\textsc{Compas}} }
    \label{tab:ml-performance-compas}
\end{table*}

\begin{table*}[t]
    \centering
    \begin{tabular}{l c c c c c c | c c c c c c }
    \toprule
 & \multicolumn{6}{c|}{gender} & \multicolumn{6}{c}{age}\\ \midrule
   & $\mathtt{acc}$  & $\mathtt{ap}$  & $\mathtt{auc}$  & $\mathtt{f1}$  & $\DP$  & $\EO$  & $\mathtt{acc}$  & $\mathtt{ap}$  & $\mathtt{auc}$  & $\mathtt{f1}$  & $\DP$  & $\EO$ \\ 
\midrule 
$\mathtt{ERM}$  & \textbf{75.5}  & \textbf{87.1}  & \textbf{77.4}  & \textbf{82.7}  & 5.3  & 5.5  & 74.7  & \textbf{86.3}  & \textbf{76.1}  & 82.2  & 13.8  & 14.7 \\ 
$\mathtt{DiffDP}$  & 73.0  & 85.5  & 73.9  & 81.0  & \textbf{1.1}  & 3.5  & 71.7  & 85.8  & 74.7  & 80.3  & \textbf{0.5}  & \textbf{1.8} \\ 
$\mathtt{HSIC}$  & 72.7  & 85.5  & 73.8  & 81.1  & 1.4  & \textbf{1.1}  & 73.7  & 85.7  & 74.3  & 81.9  & 4.2  & 6.0 \\ 
$\mathtt{LAFTR}$  & 73.0  & 87.0  & 75.8  & 81.1  & 8.2  & 4.1  & 73.0  & 85.2  & 73.8  & 81.1  & 10.0  & 6.4 \\ 
$\mathtt{PR}$  & 72.7  & 85.5  & 73.2  & 80.9  & 5.7  & 4.1  & \textbf{75.0}  & 85.7  & 74.8  & \textbf{82.8}  & 6.5  & 4.8 \\ 
\bottomrule
    \end{tabular}
    \caption{Performance of the ML models. Dataset: \texttt{German} }
    \label{tab:ml-performance-german}
\end{table*}

\subsection{Approximation of the Input Distribution}
\label{sec:cost_alignment}

For shield synthesis, we need a distribution of the input space $\sgen\in \distrset(\aG\times\BB\times\costset)$. 
In the ideal case, $\sgen\in \distrset(\aG\times\BB\times\costset)$ is the exact joint distribution of group membership, agent recommendation and cost. 
However, this is unrealistic most of the time, as it assumes knowledge of the underlying distribution and the classifier. 
Furthermore, the distribution of cost given by the agent may be continuous, but we assume that there is a finite set $\costset$ of costs allowed. 

For our experiments, we used 
a simple approach that is agnostic to the ML classifier.

In our approach, we assume that any recommendation is equally likely, and the cost constant.
Therefore $\costset = \{1\}$, and for all $g\in\aG$ and $r\in \BB$, 
$\sgen (g, b, 1) = p_g/2$.

As discussed, this approximation is easy to compute and agnostic to the ML classifier.

% \subsubsection{Input distribution paired with the ML classifier.}

% In this case, we assume access to the underlying population distribution and the agent. 
% This is the case if we have, for example, a training fragment of the dataset and the weights of the neural network that define the classifier. 

% In such cases, we can learn a histogram of costs for each group-decision pairing. 
% To do so, first, we define 
% the limits of the bins
% $0 = \lambda_0 \leq \lambda_1\leq \dots \leq \lambda_{n-1}  \leq \lambda_{n} = 0.5$.
% Let $\Lambda_i = [\lambda_{i-1}, \lambda_i]$ for $i\in\{1,\dots,n\}$, and $c_i = (\lambda_i + \lambda_{i-1})/2$.
% The cost set is $\costset = \{c_1, \dots, c_n\}$
% as a discretization of the set $[0, 1/2]$ in $n$ bins.

% For each $g \in \aG$ and $r \in \{0, 1\}$, 
% let $N_{g, r}$ be the total number of instances in $\traindata$ of group $g$ for which the agent recommends $r$.
% For each $i\in\{1,\dots,n\}$, let $N_{gri}$ be the number of instances in $\traindata$ of group $g$ in which the agent recommends $g$ and the cost of changing is inside $\Lambda_i$. Formally, this is
% \[
% N_{gri} = \hspace{-1.5em} \sum_{\left(f',g'\right)\in\traindata} \hspace{-1em} 
% \indicator{r = \left(\rho(f,g) > {\textstyle \frac{1}{2}}\right) \land \left|\rho(f,g) - {\textstyle \frac{1}{2}}\right|\in \Lambda_i}.
% \]
% Then the probability distribution is defined as
% \[
%     \theta(g,r,c_i) = \frac{N_{g,r,i}}{N_{g, r}}.
% \]
% In our experiments, we used $n=10$ bins.

    \section{Extended Experimental Evaluation}
    \label{sec:experiments_appendix}
    In this appendix, we present further experimental results and insights that we could not fit into the main paper.

\subsection{Shield Synthesis Computation Times}
\label{sec:shield-computation-times}

As pointed out in Thm.~\ref{thm:bounded-horizon shield synthesis-bis}, 
our shield synthesis algorithm has a polynomial complexity, and the degree of the polynomial is the number 
of counters required to keep track of the fairness property. 

In our experimental evaluation we measured and enforced two fairness properties: $\DP$ and $\EO$.
For $\DP$ it is sufficient to track 4 counters: the number of instances appeared and accepted of each group. 
For $\EO$, we also need 4 counters for the number of instances appeared and accepted of each group, counting only those for which $z=1$. Furthermore, we need two extra counters: one to count all instances with $z =0$, and one to keep track of the last decision for which ground truth has not yet been revealed, for a total of 6 counters. 

In Figure~\ref{fig:comp_resources} we show the computation time and memory usage of our shield synthesis algorithm for a fixed problem with increasing time horizon. Fig.~\ref{fig:comp_resources} does not show variability, because the synthesis algorithm, as described, is deterministic.

\begin{figure}
     \centering
     \begin{subfigure}[b]{0.48\linewidth}
         \centering
         \includegraphics[width=\linewidth]{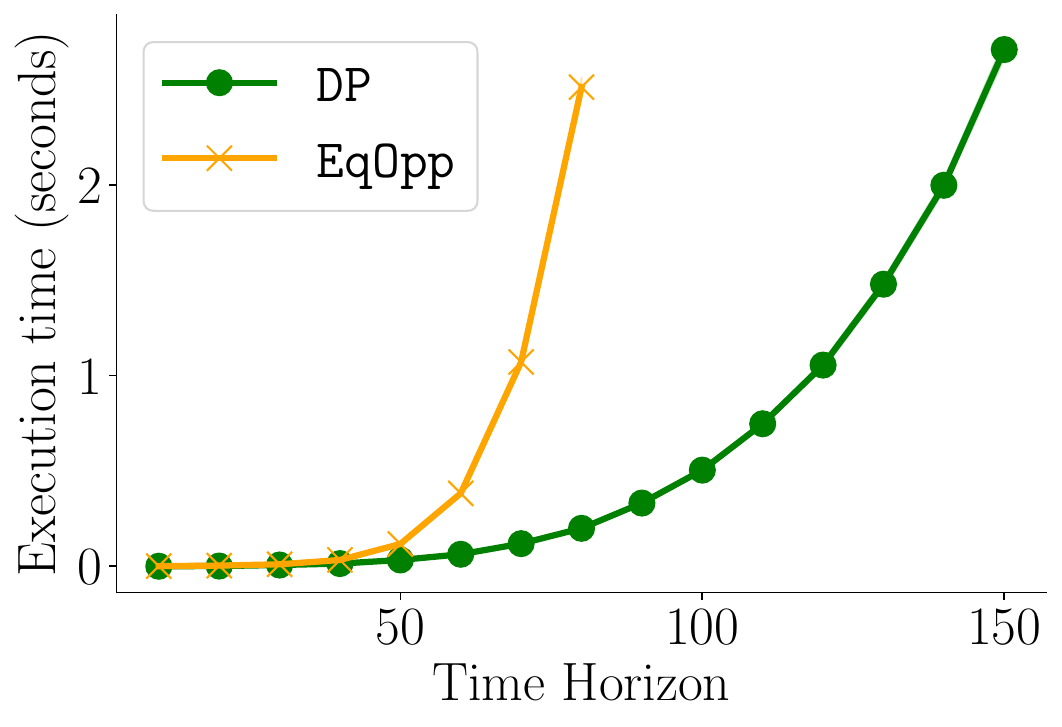}
         \caption{Time}
         \label{fig:comp_time}
     \end{subfigure}
     \hfill
     \begin{subfigure}[b]{0.48\linewidth}
         \centering
         \includegraphics[width=\linewidth]{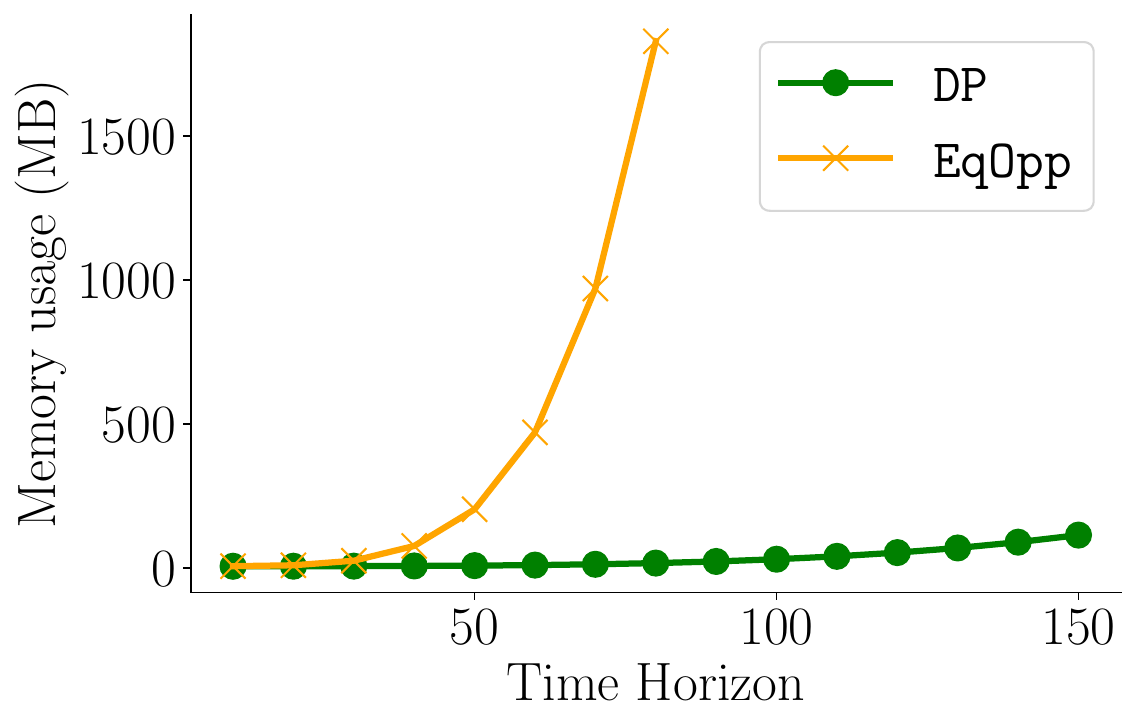}
         \caption{Memory}
         \label{fig:comp_memory}
     \end{subfigure}
        \caption{Resource usage for shield synthesis with increasing time horizons.}
        \label{fig:comp_resources}
\end{figure}

\subsection{\FinShield Shields}

\paragraph{Fairness.}
Fig.~\ref{fig:fair-distr} and Tab.~\ref{tab:fair_distr} describe the distribution of the normalized bias with and without shielding, i.e., the bias divided by the respective threshold. Here we can observe that a good portion of the runs without shield violate the fairness constraint, by contrast we can observe that our shields manage to uphold the condition. However, we want to highlight, that most runs with shield achieve a fairness value significantly below the threshold. This is more pronounced for EqOpp.

\paragraph{Utility loss.}
The focus on utility loss in the main body is justified by the strong relationship between cost and utility loss observed in Fig~\ref{fig:cost-utility-corr}. 
From the perspective of utility loss the observations highlighted in the main body of the paper remain the same. 
That is, from Fig~\ref{fig:utility-comparison-0.15} and Fig~\ref{fig:utility-comparison-0.2} we can observe that DP incurs less utility loss than EqOpp and that most of the variability comes from the dataset rather than from the ML Algorithms.
This is supported by Fig.\ref{fig:utility_environment_box} and Fig.~\ref{fig:utility_mlalgo_box}, which provide insight into the distribution of utility loss for each dataset and ML Algorithm respectively. The most striking difference is that when compared to the datasets, the utility loss distribution varies only slightly between ML Algorithms.

\begin{figure*}%[t]
   \centering
   \includegraphics[width=\linewidth]{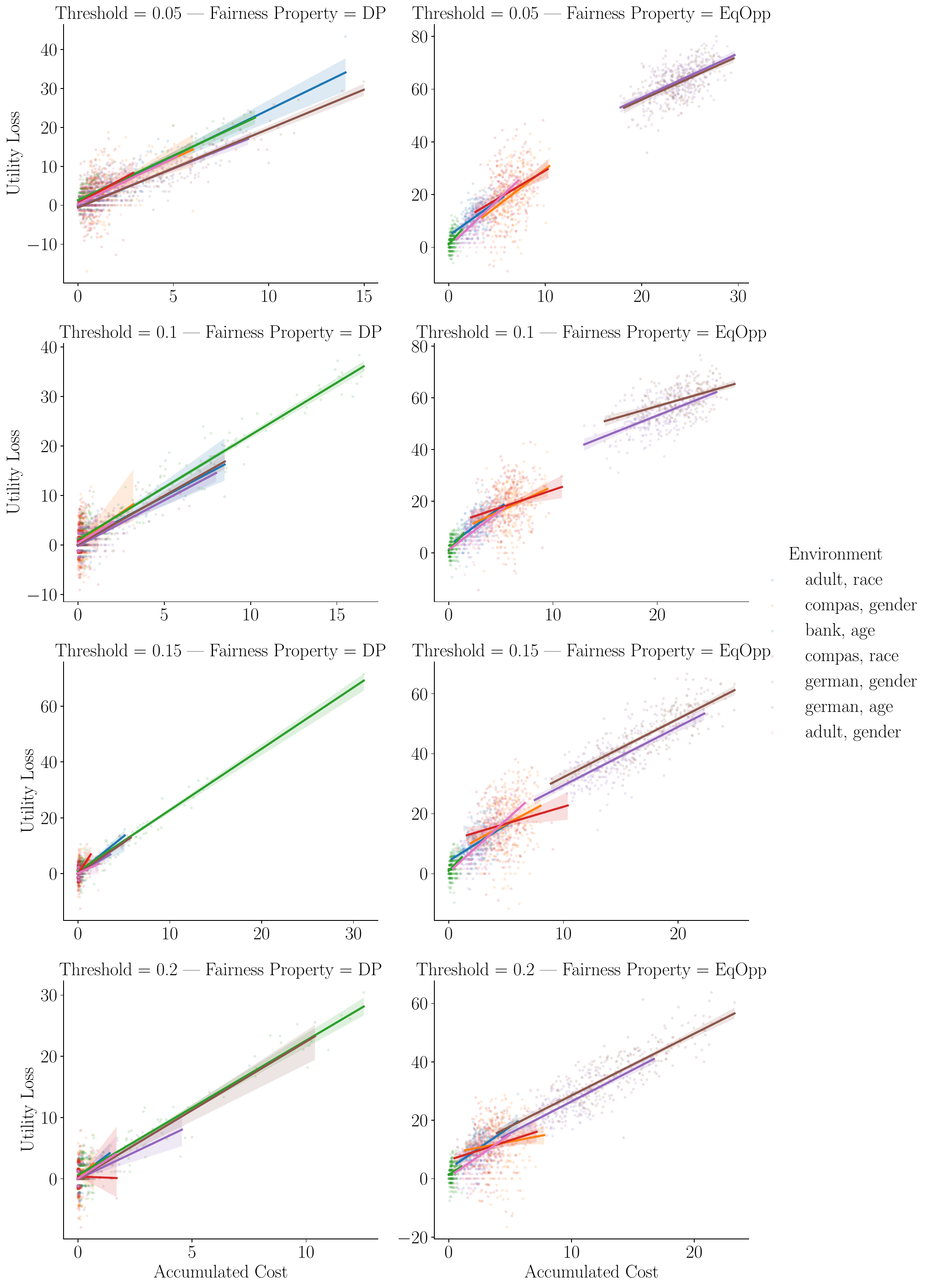}  %
    \caption{Regression plot depicting the relationship between utility loss and cost for various $\kappa$ for each dataset. DP (left) and EqOpp (right)}
    \label{fig:cost-utility-corr}
\end{figure*}

\begin{table*}[t]
\scriptsize
    \centering
\begin{tabular}{ll|rrr|rr|r}
\toprule
 &  & 25\% Quantil  & Median & 75\% Quantil & Mean & Std & Above \\
\midrule
\multirow[c]{2}{*}{DP} & No Shield & 0.38 & 0.83 & 1.59 & 1.22 & 1.29 & 42.46 \% \\
 & Static-Fair & 0.18 & 0.42 & 0.74 & 0.46 & 0.31 & 0.00\% \\
\multirow[c]{2}{*}{EqOpp} & No Shield & 0.67 & 1.76 & 3.62 \%& 2.76 & 3.01 & 65.06 \%\\
 & Static-Fair & 0.00 & 0.21 & 0.50 & 0.27 & 0.28 & 0.00 \%\\
 \bottomrule
\end{tabular}

    \caption{Statistic of normalized fairness, i.e., fairness value / threshold}
    \label{tab:fair_distr}
\end{table*}

\begin{table*}%[t]
   \centering
   \includegraphics[width=\linewidth]{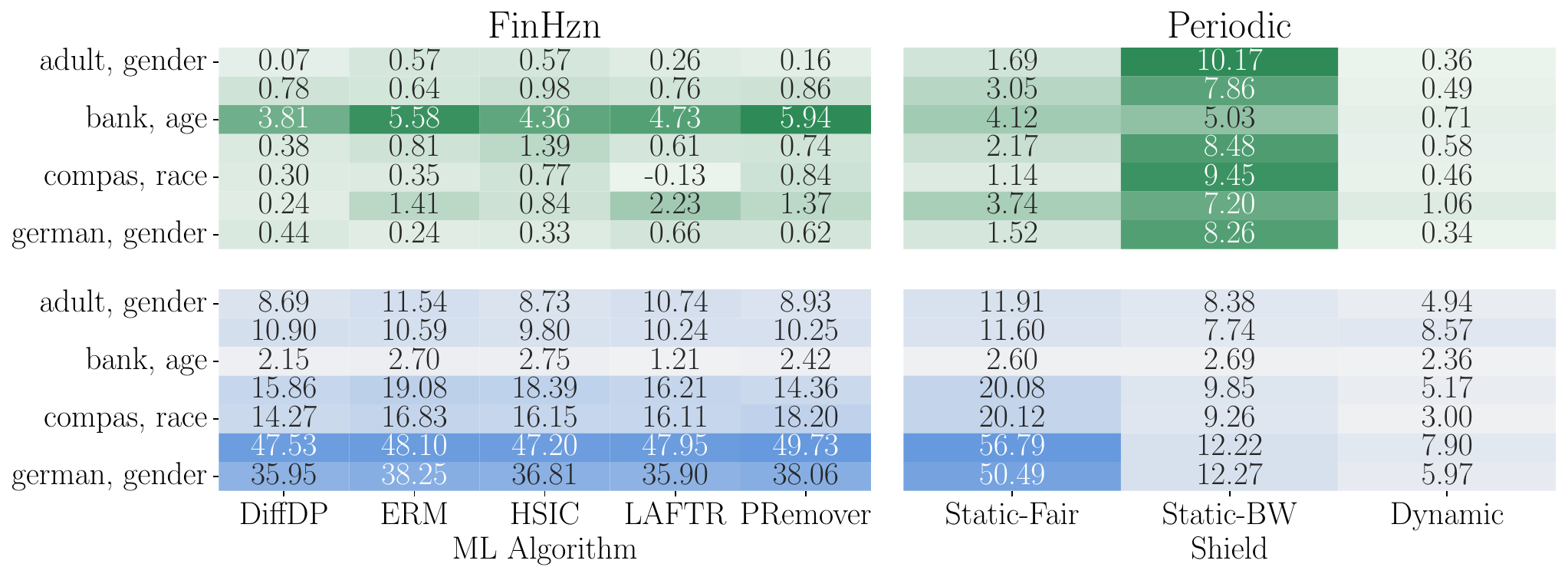}  %
    \caption{Utility loss (in \%) incurred by \FinShield shields for different ML models (left) and by periodic shields on the ERM model (right) for the fairness properties DP (top, green) and EqOpp (bottom, blue) with $\kappa=0.15$. Lighter colors indicate smaller utility loss.}
    \label{fig:utility-comparison-0.15}
\end{table*}

\begin{table*}%[t]
   \centering
   \includegraphics[width=\linewidth]{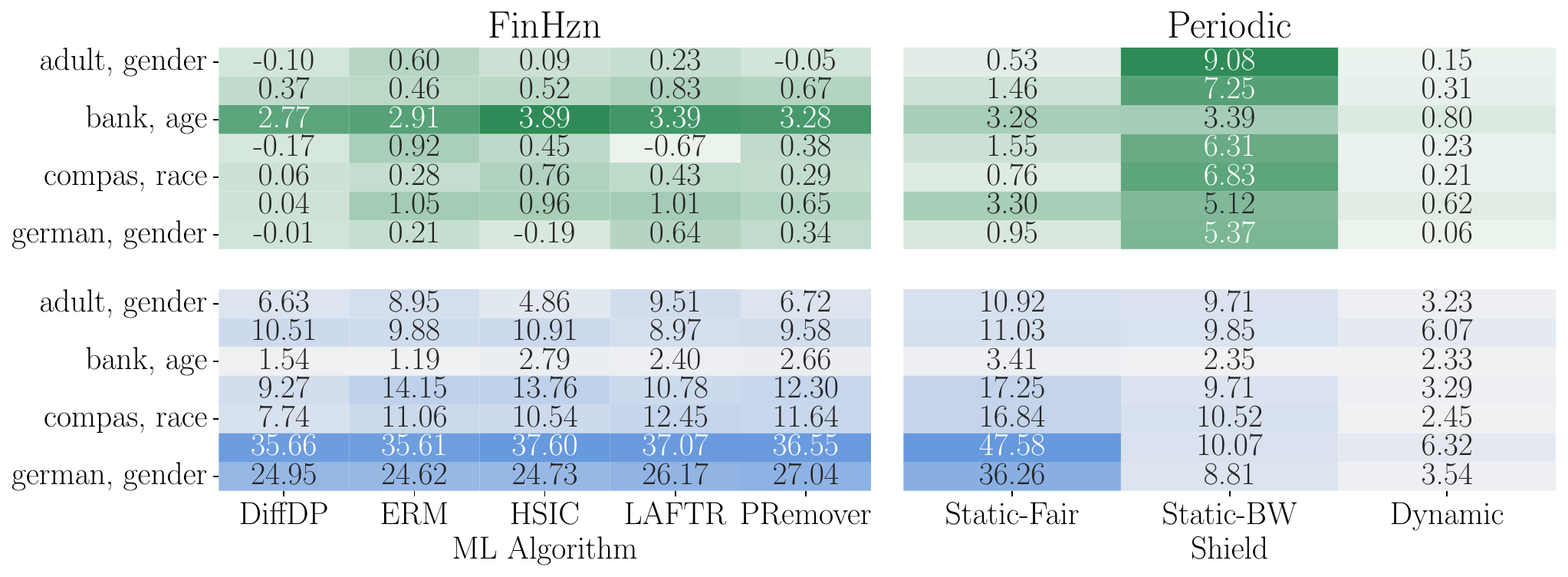}  %
    \caption{Utility loss (in \%) incurred by \FinShield shields for different ML models (left) and by periodic shields on the ERM model (right) for the fairness properties DP (top, green) and EqOpp (bottom, blue) with $\kappa=0.2$. Lighter colors indicate smaller utility loss.}
    \label{fig:utility-comparison-0.2}
\end{table*}

\begin{figure*}%[t]
   \centering
   \includegraphics[width=\linewidth]{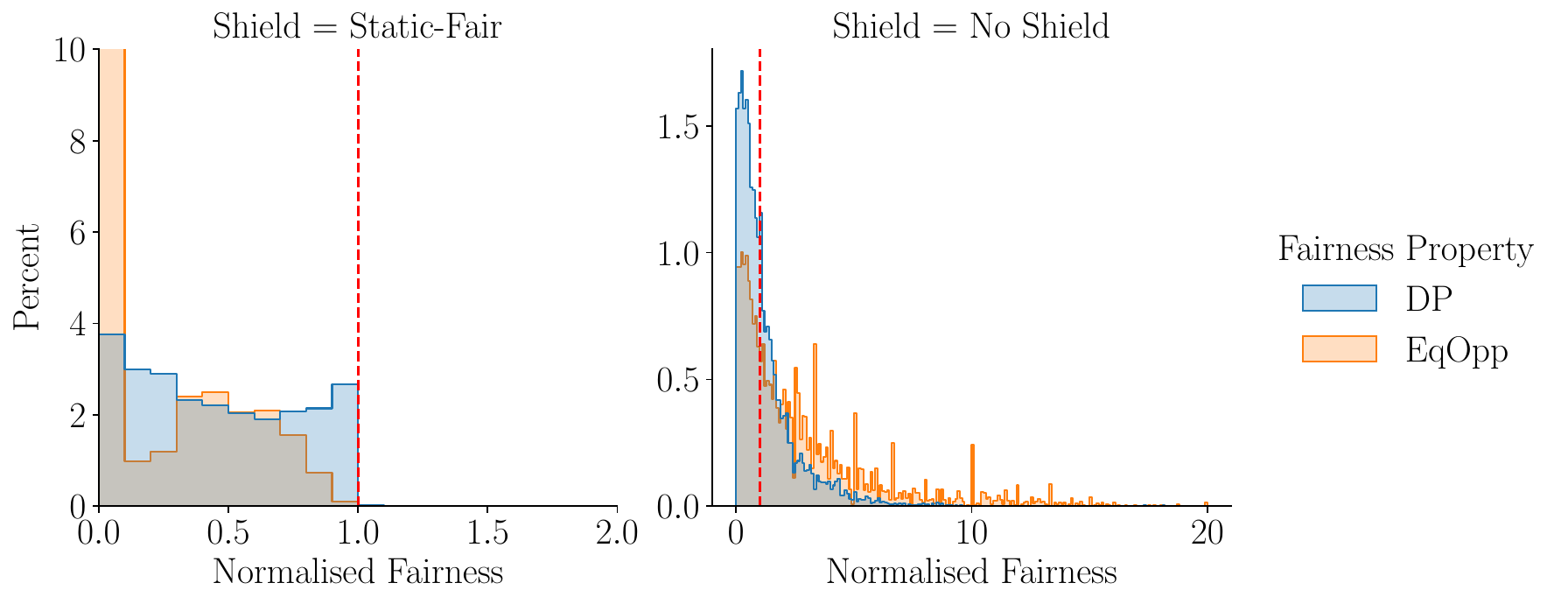}  %
    \caption{Distribution of normalized bias, i.e. Bias / $\kappa$, across all runs with (left) and without shield (right) for both DP and EqOpp.}
    \label{fig:fair-distr}
\end{figure*}

\begin{figure*}%[t]
   \centering
   \includegraphics[width=\linewidth]{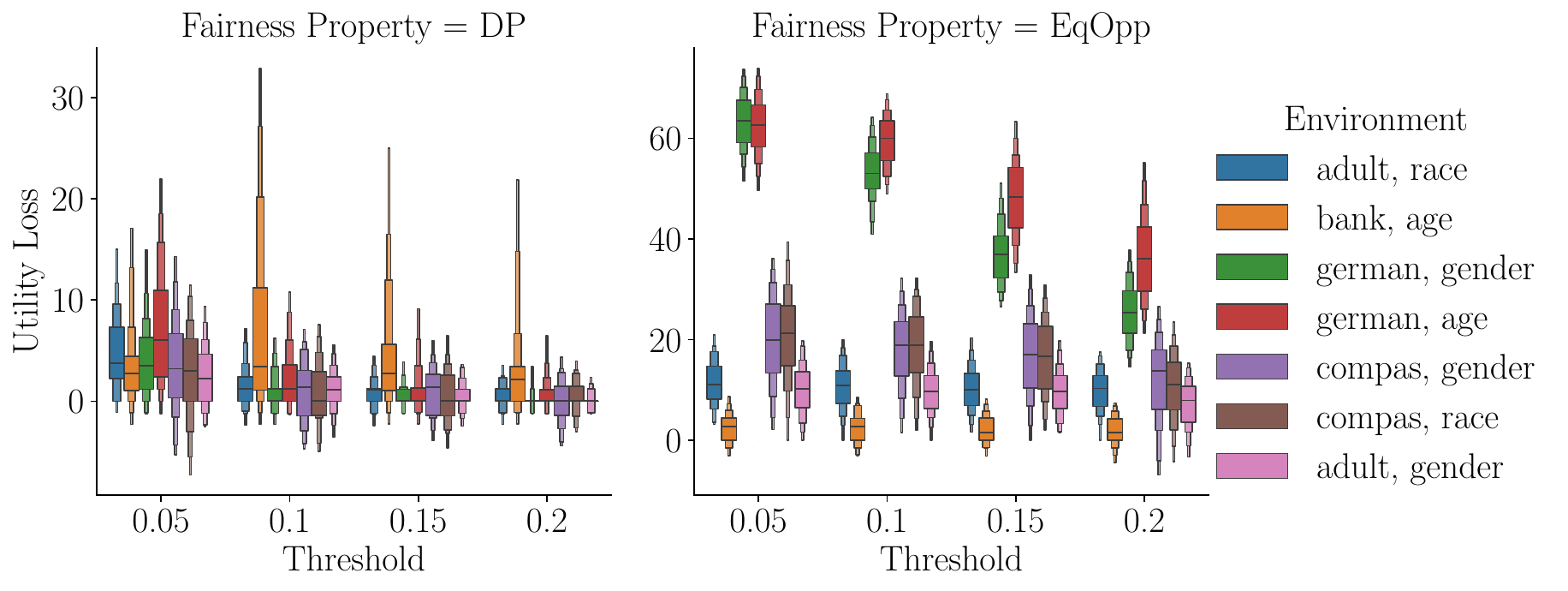}  %
    \caption{Distribution of utility loss (in $\%$) incurred by \FinShield aggregated across all ML Algorithms for DP (left) and EqOpp (right). The hight of the boxes indicate the spread of the distribution. }
    \label{fig:utility_environment_box}
\end{figure*}

\begin{figure*}%[t]
   \centering
   \includegraphics[width=\linewidth]{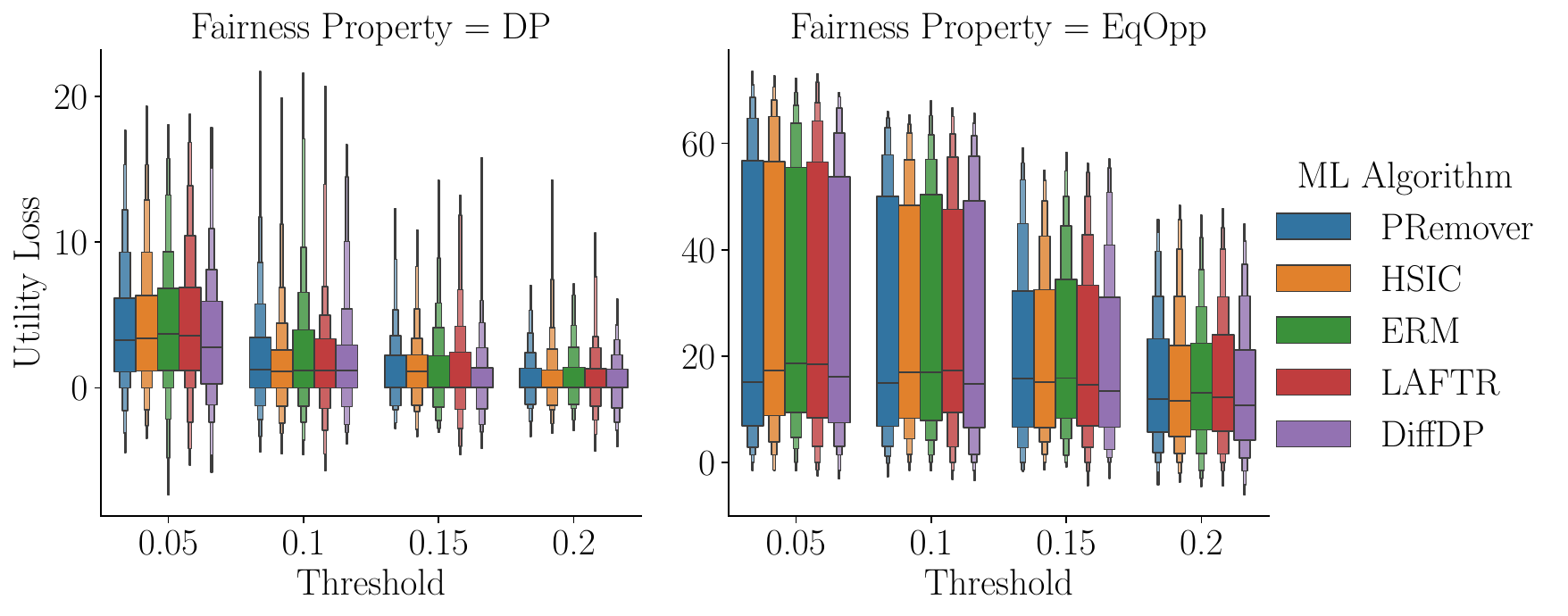}  %
    \caption{Distribution of utility loss (in $\%$) incurred by \FinShield aggregated across all environments for DP (left) and EqOpp (right). The hight of the boxes indicate the spread of the distribution. }
    \label{fig:utility_mlalgo_box}
\end{figure*}

\subsection{Periodic Shielding}
All the experiments are evaluated using the ERM ML Algorithm only. 

\paragraph{Fairness.}
Both Fig.~\ref{fig:fair-period} and Fig.~\ref{fig:fair-period-cond}
provide insight into the distribution of the normalized bias. 
The former aggregates over all runs, while the latter considers only those runs for \StaticBAR and \Dyn that satisfy the assumption. We can observe that \StaticBAR has a relatively high rate of violation in Fig.~\ref{fig:fair-period}, while at the same time being overly conservative in Fig.~\ref{fig:fair-period-cond}.
This problem does not exist with \Dyn as the median run is only slightly below the threshold. Moreover, Fig.~\ref{fig:fair-period-cond} indicates that no run satisfied the assumption of \StaticDP.
However, most importantly in all cases we can report an improvement over the unshielded runs.

\paragraph{Utility.}
In Fig~\ref{fig:utility-comparison-0.15} and Fig~\ref{fig:utility-comparison-0.2} we can observe the utility loss for $\kappa=0.15$ and $\kappa=0.2$ respectively.
We observe that the utility loss is low for both \StaticDP and \Dyn when compared to \StaticBAR and decreases with an increase in $\kappa$.
Fig.~\ref{fig:utility-periods} we depict the distribution of the percentage utility loss across periods. That is, for each run we normalize the utility loss per period by the total utility loss of the run. We can observe that for \Dyn the utility loss decreases in later periods. 
A trend not observed for the other shields. 
Moreover, we can observe that in some periods the shield actually increases the utility, i.e., we can observe a negative utility loss, which leads to some periods having above $100\%$ utility loss.

\begin{figure*}%[t]
   \centering
   \includegraphics[width=\linewidth]{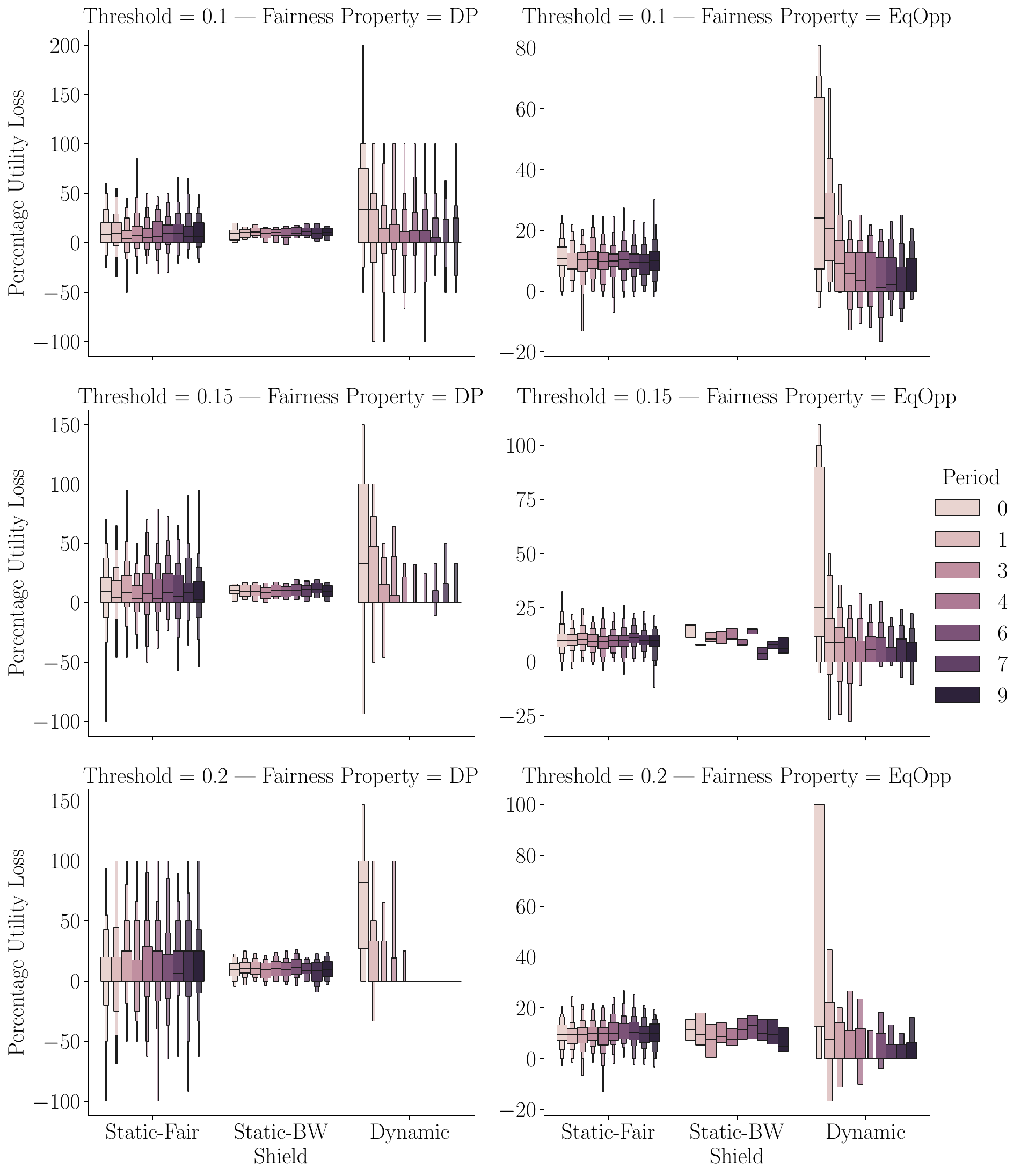}  %
    \caption{Percentage of total utility loss (in $\%$) for each period incurred by \StaticDP, \StaticDP and \Dyn across all runs for each $\kappa$. 
    DP (left) and EqOpp (right)}
    \label{fig:utility-periods}
\end{figure*}

\begin{figure*}%[t]
   \centering
   \includegraphics[width=\linewidth]{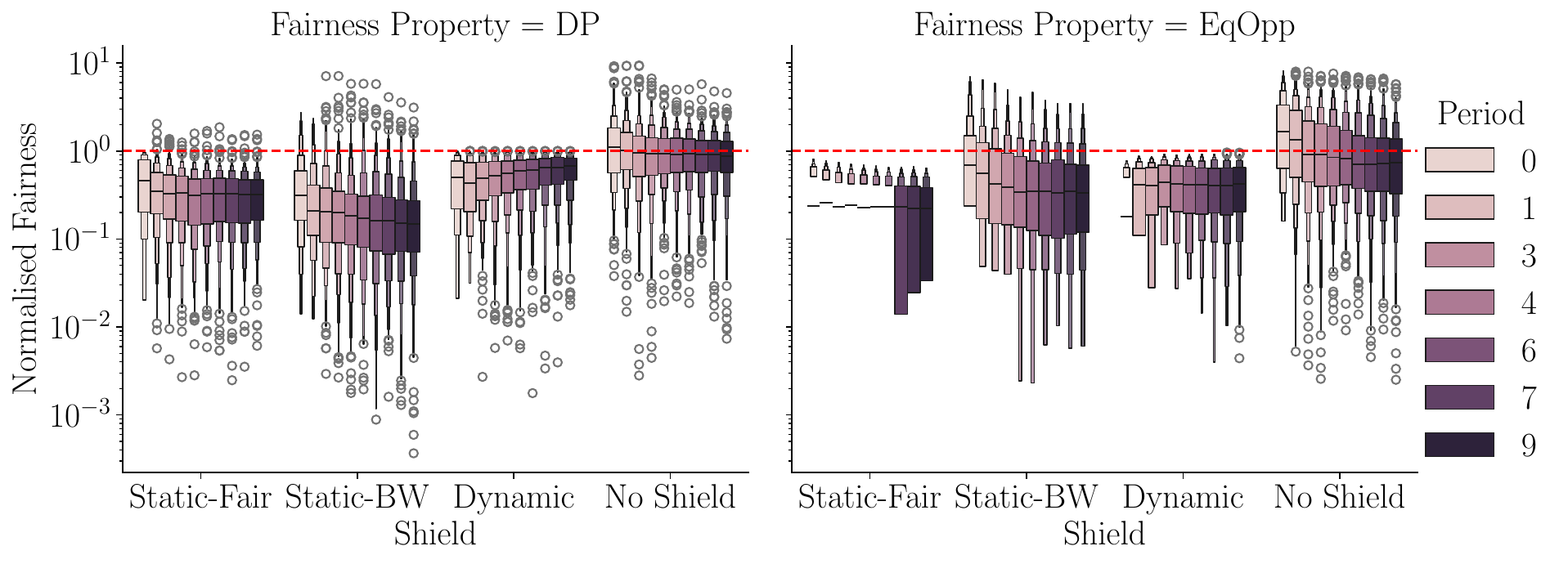}  %
    \caption{Distribution of normalized bias, i.e. Bias / $\kappa$, for each period for all runs. Each run below the red line satisfies the fairness condition. DP (left) and EqOpp (right)}
    \label{fig:fair-period}
\end{figure*}

\begin{figure*}%[t]
   \centering
   \includegraphics[width=\linewidth]{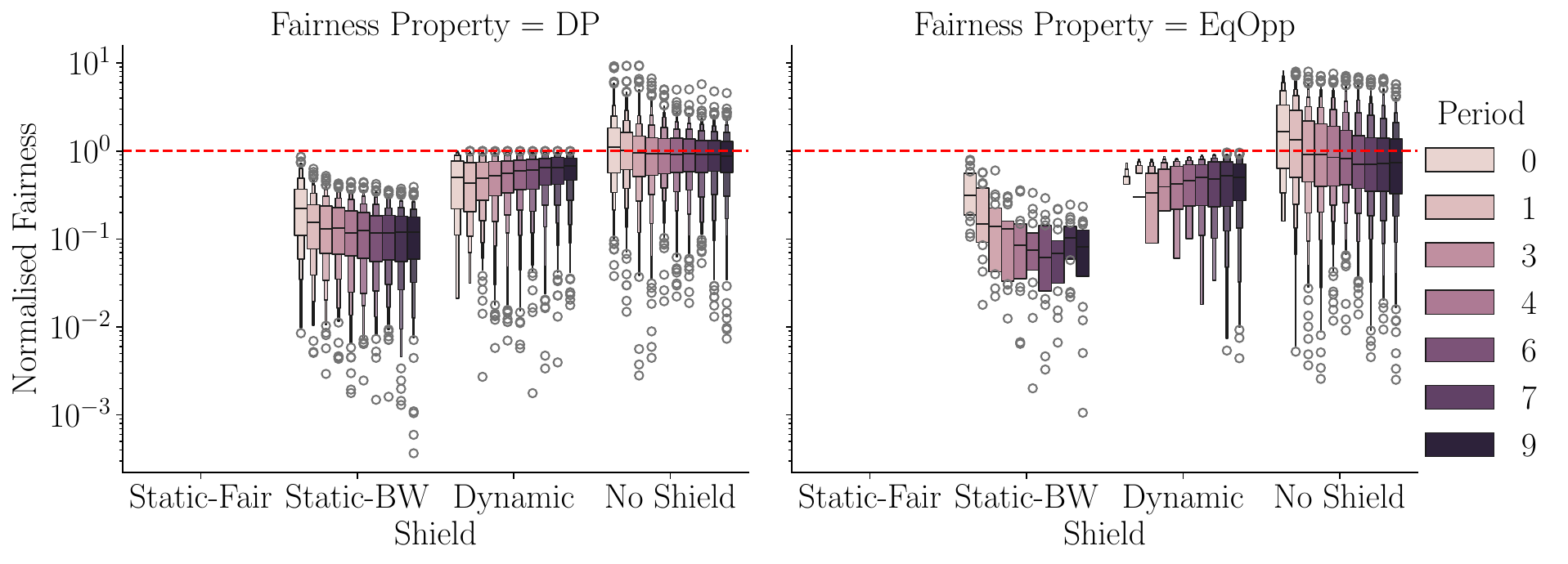}  %
    \caption{Distribution of normalized bias, i.e. Bias / $\kappa$, for each period for all runs \emph{where the assumption is satisfied}. Each run below the red line satisfies the fairness condition. DP (left) and EqOpp (right).}
    \label{fig:fair-period-cond}
\end{figure*}

}{
    % If appendices are not included, this part is skipped.
}

\clearpage

% \input{80_paper_structure_notes}
% \clearpage

% \input{83_Old_shield_definitions}
% \clearpage

% \input{85_compositional_full_discussion}

% \clearpage
% \section{Algorithm for the bounded-horizon problem (safekeeping)}
% \input{35_synthesis}

% \section{Algorithms for the Periodic Problem (safekeeping)}
% \input{40_unbounded}

% \clearpage
% \section{Macros cheatsheet}
% \input{87_macros_corner}

% \includepdf[pages=1-10]{full_document.pdf}
\end{document}